\documentclass{article}

\usepackage[final,nonatbib]{neurips_2026}
\usepackage[numbers]{natbib}

\usepackage[page,header]{appendix} 
\usepackage{titletoc}              

\newcommand{\customsubsection}[1]{%
  \par
  \pagebreak[2]%
  \refstepcounter{subsection}%
    \everypar={%
      {\setbox0=\lastbox}
      \addcontentsline{toc}{subsection}{%
        {\protect\makebox[0.3in][r]{\thesubsubsection.} \hspace*{3pt}#1}}%
      \textbf{\thesubsection\space\space{#1}\space\newline}%
      \everypar={}%
    }%
  \ignorespaces
}

\usepackage{xcolor}
\usepackage{tcolorbox}
\tcbuselibrary{breakable,skins}

\newcommand{\corestep}[1]{%
\begingroup
\setlength{\fboxsep}{1.5pt}%
\colorbox{green!13}{\parbox{0.965\linewidth}{#1}}%
\endgroup
}

\newcommand{\redstep}[1]{%
\begingroup
\setlength{\fboxsep}{1.5pt}%
\colorbox{red!10}{\parbox{0.965\linewidth}{#1}}%
\endgroup
}

\newcommand{\neutralstep}[1]{%
\begingroup
\setlength{\fboxsep}{1.5pt}%
\colorbox{gray!7}{\parbox{0.965\linewidth}{#1}}%
\endgroup
}

\newtcolorbox{qualtracebox}[1]{
  enhanced,
  breakable,
  colback=white,
  colframe=gray!45,
  boxrule=0.35pt,
  arc=1pt,
  left=4pt,
  right=4pt,
  top=4pt,
  bottom=4pt,
  title={#1},
  fonttitle=\bfseries\small,
  colbacktitle=gray!12,
  coltitle=black,
  before skip=5pt,
  after skip=5pt
}

\newcommand{\quallegend}{%
\noindent
\colorbox{green!13}{\strut\hspace{0.8em}}~Minimal core step
\quad
\colorbox{red!10}{\strut\hspace{0.8em}}~Removed redundant step
\quad
\colorbox{gray!7}{\strut\hspace{0.8em}}~Input / answer metadata
}

\usepackage[utf8]{inputenc} 
\usepackage[T1]{fontenc}    
\usepackage{hyperref}       
\usepackage{url}            
\usepackage{booktabs}       
\usepackage{amsfonts}       
\usepackage{nicefrac}       
\usepackage{microtype}      
\usepackage{xcolor}         
\usepackage{enumitem}

\usepackage{algorithm}
\usepackage{algorithmic}
\usepackage{amsmath}
\usepackage{amssymb}
\usepackage{cleveref}
\usepackage{graphicx}

\usepackage{algorithm}
\usepackage{algorithmic}
\usepackage{amsmath}
\usepackage{amssymb}
\usepackage{amsthm}
\usepackage{cleveref}
\newcommand{\secondbest}[1]{\underline{#1}}

\theoremstyle{definition}

\usepackage{booktabs}
\usepackage{multirow}
\usepackage{makecell}
\usepackage{array}
\usepackage{graphicx}
\usepackage{xcolor}
\usepackage{colortbl}
\usepackage{subcaption}
\usepackage{pgfplots}
\pgfplotsset{compat=1.18}

\newcolumntype{L}[1]{>{\raggedright\arraybackslash}p{#1}}

\definecolor{coreblue}{RGB}{226,239,255}
\definecolor{coreblueDark}{RGB}{30,92,160}
\definecolor{coregray}{RGB}{245,245,245}
\definecolor{coreorange}{RGB}{224,120,57}
\definecolor{coregreen}{RGB}{50,145,90}

\newcommand{\corecell}[1]{\cellcolor{coreblue}#1}
\newcommand{\corecellbf}[1]{\cellcolor{coreblue}\textbf{#1}}

\pgfplotsset{
    greedyplot/.style={coreblueDark, very thick, mark=*, mark size=1.7pt},
    necessityplot/.style={coreorange, thick, dashed, mark=square*, mark size=1.5pt},
    randomplot/.style={black!45, thick, dotted, mark=triangle*, mark size=1.7pt},
    densegrid/.style={
        grid=both,
        grid style={black!7},
        tick label style={font=\scriptsize},
        label style={font=\scriptsize},
        title style={font=\scriptsize}
    }
}

\usepackage[table]{xcolor}

\usepackage{booktabs}
\usepackage{multirow}
\usepackage{makecell}
\usepackage{adjustbox}
\usepackage{bbm}

\tcbuselibrary{theorems,breakable,skins}

\usepackage{tikz}
\usetikzlibrary{calc}
\usepackage{pgfplots}
\pgfplotsset{compat=1.18}
\usepgfplotslibrary{fillbetween,groupplots,colormaps}

\tcbset{
  theoremstyle/.style={
    enhanced,
    breakable,
    colback=blue!2,
    colframe=blue!35!black,
    coltitle=white,
    colbacktitle=blue!55!black,
    boxrule=0.5pt,
    arc=1.5pt,
    left=4pt,
    right=4pt,
    top=4pt,
    bottom=4pt,
    before skip=6pt,
    after skip=6pt,
    fonttitle=\bfseries
  },
  proofstyle/.style={
    enhanced,
    breakable,
    colback=gray!4,
    colframe=gray!40,
    boxrule=0.4pt,
    arc=1.5pt,
    left=4pt,
    right=4pt,
    top=4pt,
    bottom=4pt,
    before skip=4pt,
    after skip=6pt
  },
  defstyle/.style={
    enhanced,
    breakable,
    colback=blue!1,
    colframe=blue!25!black,
    coltitle=white,
    colbacktitle=blue!45!black,
    boxrule=0.35pt,
    arc=1pt,
    left=4pt,
    right=4pt,
    top=3pt,
    bottom=3pt,
    before skip=4pt,
    after skip=4pt,
    fonttitle=\bfseries\small
  }
}

\newtcbtheorem[number within=section]{theorem}{Theorem}{theoremstyle}{thm}
\newtcbtheorem[number within=section]{lemma}{Lemma}{theoremstyle}{lem}
\newtcbtheorem[number within=section]{proposition}{Proposition}{theoremstyle}{prop}
\newtcbtheorem[number within=section]{definitionbox}{Definition}{defstyle}{defn}

\newcommand{\stdgreen}[1]{\ensuremath{{\color{green!45!black}\scriptstyle \pm #1}}}

\newtcolorbox{appthmbox}[1]{
  theoremstyle,
  title={Theorem: #1}
}

\newtcolorbox{applembox}[1]{
  theoremstyle,
  title={Lemma: #1}
}

\newtcolorbox{apppropbox}[1]{
  theoremstyle,
  title={Proposition: #1}
}

\newtcolorbox{compactdefbox}{
  enhanced,
  breakable,
  colback=blue!2,
  colframe=blue!25!black,
  boxrule=0.35pt,
  arc=1pt,
  left=4pt,
  right=4pt,
  top=3pt,
  bottom=3pt,
  before skip=4pt,
  after skip=5pt
}

\newtcolorbox{takeawaybox}{
  enhanced,
  breakable,
  colback=gray!4,
  colframe=gray!35,
  boxrule=0.35pt,
  arc=1pt,
  left=4pt,
  right=4pt,
  top=2pt,
  bottom=2pt,
  before skip=4pt,
  after skip=4pt
}

\newtcolorbox{apptakeawaybox}{
  enhanced,
  breakable,
  colback=gray!4,
  colframe=gray!35,
  boxrule=0.30pt,
  arc=1pt,
  left=4pt,
  right=4pt,
  top=2pt,
  bottom=2pt,
  before skip=4pt,
  after skip=4pt
}

\providecommand{\corecell}[1]{\cellcolor{blue!10}{#1}}
\providecommand{\corecellbf}[1]{\cellcolor{blue!15}{\textbf{#1}}}
\providecommand{\appsmalltable}{\small\setlength{\tabcolsep}{4pt}\renewcommand{\arraystretch}{1.04}}

\title{Uncovering the Representation Geometry of Minimal Cores in Overcomplete Reasoning Traces}

\vspace{-18pt}
\author{%
  Sanjoy Chowdhury\thanks{Work done while the author was at UMD} \qquad   Dinesh Manocha \\
  University of Maryland, College Park, USA\\
  \texttt{\{sanjoyc, dmanocha\} @umd.edu} \\
}

\makeatletter
\providecommand{\@trackname}{}
\makeatother
\begin{document}

\maketitle
\begin{abstract}
\vspace{-4pt}
Language models often generate long chain-of-thought traces, but it remains unclear how much of this reasoning is necessary for preserving the final prediction. We study this through the lens of \emph{overcomplete reasoning traces}: generated traces that contain more intermediate steps than are needed to support the model’s answer. We define the \emph{minimal core} as the smallest subset of steps that preserves either the final answer or predictive distribution, and introduce metrics for compression ratio, redundancy mass, step necessity, and necessity concentration. Across six deliberative reasoning benchmarks spanning arithmetic, competition mathematics, expert scientific reasoning, and commonsense multi-hop QA, we find substantial overcompleteness: on average, \textbf{46\%} of steps are removable under greedy minimal-core extraction while preserving the original answer in \textbf{86\%} of cases. We also find that predictive support is concentrated: the top three steps account for \textbf{65\%} of measured necessity mass on average. Beyond compression, minimal cores expose a cleaner geometry of reasoning: compared with full traces, they improve correct--incorrect trace separation by \textbf{11 points}, reduce estimated intrinsic dimensionality by \textbf{34\%}, and transfer across model families with \textbf{85\%} off-diagonal answer retention. Theoretically, we establish existence of minimal sufficient subsets, local irreducibility guarantees for greedy elimination, and certificates of overcompleteness and sparse necessity. Together, these results suggest that full reasoning traces are often verbose and overcomplete, while minimal cores isolate the effective support underlying language-model predictions.
\end{abstract}

\section{Introduction}
\label{sec:introduction}
\vspace{-4pt}

Language models increasingly use explicit intermediate reasoning traces such as chain-of-thought derivations, plan-style decompositions, and step-by-step explanations to solve complex reasoning problems~\citep{wei2022chain,kojima2022large,zhou2023least,yao2023tree}. These traces are now central to prompting, verification, process supervision, and evaluation~\citep{wang2023selfconsistency,lightman2023lets}. Yet a generated trace is not necessarily a minimal explanation \cite{su2025between, eisenstadt2025overclocking}: it may include restatements, paraphrases, stylistic justifications, redundant checks, or alternative routes that do not affect the final prediction. Fig.~\ref{fig:teaser} summarizes our view: full traces often contain a compact predictive core surrounded by removable reasoning mass.

\begin{figure*}[t]
    \centering
    \includegraphics[width=\columnwidth]{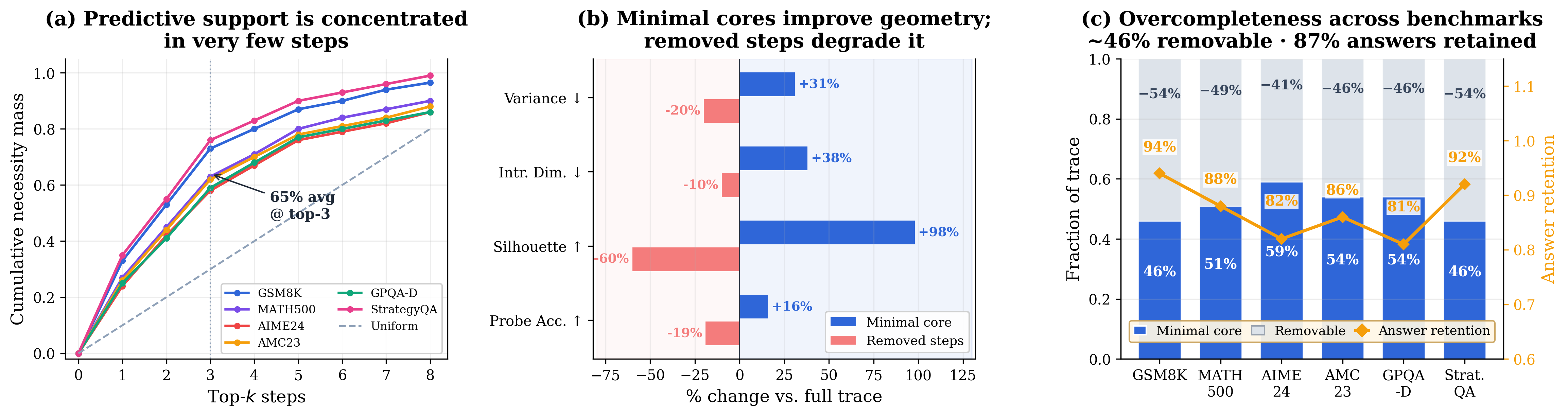}
    \caption{\textbf{LLM reasoning traces are substantially overcomplete.}
Generated traces contain a sparse predictive core surrounded by removable reasoning mass. Minimal-core extraction separates answer-supporting steps from redundant elaboration and exposes a compact, transferable representation geometry of reasoning.}
    \label{fig:teaser}   
    \vspace{-10pt}
\end{figure*}

This raises a simple question: \emph{how much of a generated reasoning trace is actually needed?} Prior work has largely focused on eliciting better reasoning \cite{feng2025efficient, qu2025survey, ma2025reasoning, hassid2025don}, supervising intermediate steps \cite{xu2025logicreward, ma2025general, jung2025code, zhao2025reason}, improving self-consistency \cite{wan2025reasoning, zhou2025theoretical, taubenfeld2025confidence, zhang2025consistent}, or training models to reason more efficiently~\citep{wei2022chain,wang2023selfconsistency,lightman2023lets,zelikman2022star}. Recent efficient-reasoning methods further show that shorter traces can sometimes preserve or improve accuracy through curated data, reinforcement learning, test-time scaling \cite{tian2025think, teng2025atom, yang2025towards, snell2025scaling}, or trimming objectives~\citep{ye2025limo,muennighoff2025s1,snell2024scaling,wang2026trims,chen2025surgical}. These directions are complementary to ours: rather than making a model generate shorter traces, we analyze traces that have already been generated and ask which steps are behaviorally sufficient for preserving the original answer or predictive distribution.

We study this problem through the lens of \emph{overcomplete reasoning traces}. A trace is overcomplete if it contains a proper subset of steps that preserves the model's predictive behavior. We define the \emph{minimal core} as the smallest such sufficient subset. This turns trace analysis into a subset-sufficiency problem and induces quantitative measures of compression ratio, redundancy mass, step necessity, and necessity concentration. Unlike human-rationale extraction, our goal is not to recover a canonical explanation; instead, we characterize the effective support of a model prediction under explicit reasoning context.

Minimal cores are useful for both interpretation and evaluation. Accuracy indicates whether the final answer is correct, but not which parts of a trace support that answer. Full-trace evaluation can conflate essential reasoning with removable elaboration. Minimal cores instead expose the \emph{effective support} of a prediction, allowing us to test whether reasoning utility is distributed across many steps or concentrated in a sparse set of high-necessity steps.

Beyond compression, we ask whether minimal cores reveal a cleaner \emph{representation geometry} of reasoning, by which we mean the compactness, correct--incorrect separability, intrinsic dimensionality, and transfer behavior of trace embeddings. Representation-analysis work shows that neural embeddings can encode linguistic structure, task information, truthfulness, uncertainty, and latent decision states~\citep{kornblith2019similarity,hewitt2019structural,belinkov2022probing,aghajanyan2021intrinsic,li2023inference,marks2023geometry}. We extend this perspective to generated reasoning traces: if full traces mix answer-supporting reasoning with redundant stylistic variation, then minimal-core embeddings should be more compact, lower-dimensional, and more informative about correctness than full-trace embeddings.

We evaluate this framework on six deliberative reasoning benchmarks spanning arithmetic, competition mathematics, expert scientific reasoning, and commonsense multi-hop QA: GSM8K~\citep{cobbe2021gsm8k}, MATH500~\citep{hendrycks2021measuring,lightman2023lets}, AIME24, AMC23, GPQA-Diamond~\citep{rein2024gpqa}, and StrategyQA~\citep{geva2021strategyqa}. Across model families, greedy minimal-core extraction removes nearly half of generated steps while largely preserving the full-trace answer. Predictive support is concentrated in a small number of steps, minimal cores transfer across model families, and their embeddings are more compact and correctness-informative than full-trace or removed-step embeddings. We make the following key contributions:

\begin{itemize}[leftmargin=1.2em, labelsep=0.45em, itemsep=0pt, topsep=2pt, parsep=0pt, partopsep=0pt]
    \item \textbf{\textcolor{blue}{Trace-level minimal cores.}} We study minimal sufficiency over generated reasoning traces, defining minimal cores as the effective support of a prediction.
    \item \textbf{\textcolor{blue}{Necessity and overcompleteness.}} We quantify removable reasoning through redundancy mass, step necessity, and necessity concentration.
    \item \textbf{\textcolor{blue}{Guarantees for core extraction.}} We establish guarantees for existence, local minimality under greedy deletion, overcompleteness certificates, and sparse-necessity structure.
    \item \textbf{\textcolor{blue}{Transferable and geometric support.}} We show that minimal cores transfer across model families and occupy more compact, correctness-informative regions of representation space.
\end{itemize}

\section{Minimal Cores in Overcomplete Reasoning Traces}
\label{sec:minimal-cores}
\vspace{-4pt}

We view generated reasoning traces as structured objects containing both answer-supporting steps and redundant elaboration. A full chain-of-thought trace may include restatements, paraphrases, checks, and stylistic explanations that improve readability but are not necessarily required for preserving the model's prediction. We formalize this by defining the \emph{minimal core}: a smallest subset of generated steps that preserves the predictive behavior of the original trace.

\customsubsection{Problem Setup}
\label{subsec:problem-setup}

Let $x$ be the input and $M$ a language model. Under a prompting strategy such as chain-of-thought or plan-then-solve, the model produces a trace followed by a final answer. We write the trace as $z=(r_1,\ldots,r_T)$, where each $r_t$ is a textual reasoning unit obtained from numbered model outputs or a fixed deterministic segmentation rule. For any subset $S\subseteq[T]$, let $z_S=(r_t:t\in S)$ denote the retained subsequence in original order. Conditioning on $x$ and $z_S$, the model induces $p_M(\cdot\mid x,z_S)$ and prediction
\begin{equation}
    \hat{y}_M(x,z_S)=\arg\max_y p_M(y\mid x,z_S),
\end{equation}
with full-trace prediction $\hat{y}_M(x,z_{[T]})$. Our goal is behavioral rather than mechanistic: we do not claim that deleting written steps reveals the model's exact internal computation. Instead, we ask which generated steps are \emph{functionally necessary} for preserving the prediction when the trace is used as explicit reasoning context, giving a post-hoc way to analyze effective support without human rationales or retraining.

\customsubsection{Answer- and Distribution-Preserving Sufficiency}
\label{subsec:sufficiency}

A retained subset can be sufficient at different strictness levels. \textbf{Answer-preserving sufficiency} requires only the same final prediction: $S\subseteq[T]$ is \emph{answer-sufficient} for trace $z$ on input $x$ if $\hat{y}_M(x,z_S)=\hat{y}_M(x,z_{[T]})$. This applies even to black-box models, but it is coarse: two subsets may share the same top-1 answer while inducing different confidence profiles. \textbf{Distribution-preserving sufficiency} is stricter: given divergence $D$ and tolerance $\varepsilon\ge0$, $S$ is \emph{distribution-sufficient} if
\begin{equation}
    D\!\left(p_M(\cdot\mid x,z_S),p_M(\cdot\mid x,z_{[T]})\right)\le\varepsilon .
\end{equation}
We write $\mathsf{Suff}(S;x,z,M)\in\{0,1\}$ for a generic sufficiency predicate instantiated by either answer preservation or distribution preservation.

\customsubsection{Overcompleteness, Minimal Cores, and Step Necessity}
\label{subsec:overcompleteness}

\textbf{Minimal core.}
Given $z=(r_1,\ldots,r_T)$, a subset $S^\star\subseteq[T]$ is a \emph{minimal core} if
\begin{equation}
    S^\star\in\arg\min_{S\subseteq[T]} |S|
    \quad \text{s.t.}\quad
    \mathsf{Suff}(S;x,z,M)=1 .
    \label{eq:minimal-core}
\end{equation}
The corresponding minimal-core trace is $z^\star=z_{S^\star}$. A minimal core is the smallest retained subset preserving the chosen sufficiency criterion. It need not be unique: multiple subsets may preserve the same prediction. Our goal is therefore not to recover a canonical human rationale, but to characterize the size, necessity structure, and geometry of sufficient cores.

\textbf{Overcomplete reasoning trace.}
A reasoning trace $z$ is \emph{overcomplete} if there exists a proper subset $S\subset[T]$ such that $\mathsf{Suff}(S;x,z,M)=1$. We measure its degree using compression ratio and redundancy mass,
\begin{equation}
    \mathrm{CR}(z)=\frac{|S^\star|}{T},
    \qquad
    \mathrm{RM}(z)=1-\frac{|S^\star|}{T}.
\end{equation}
Low $\mathrm{CR}$ and high $\mathrm{RM}$ indicate that much of the trace is removable.

\begin{takeawaybox}
\noindent\textbf{Takeaway.}
Minimal cores turn reasoning-trace analysis into a subset-sufficiency problem: full traces may be verbose objects, while minimal cores expose the smallest answer-supporting subset under a specified behavioral criterion.
\end{takeawaybox}

To identify which steps matter, let $\ell(x,z_S)$ be a task loss, such as the negative log-likelihood of the full-trace answer under retained subset $z_S$. The leave-one-out necessity of step $r_t$ is
\begin{equation}
    \Delta_t(z)=
    \ell(x,z_{[T]\setminus\{t\}})-\ell(x,z_{[T]}).
    \label{eq:necessity}
\end{equation}
Large positive $\Delta_t$ means deleting $r_t$ is harmful; near-zero $\Delta_t$ suggests that the step is individually redundant under the chosen loss. This score measures \emph{marginal} necessity rather than total causal contribution: duplicate or substitutable reasoning steps may each receive low leave-one-out necessity even if the group collectively supports the answer. We therefore interpret $\Delta_t$ as unique marginal support and complement it with subset-level sufficiency tests.

When likelihoods are unavailable, we use a continuous proxy rather than the exact answer-flip indicator for ranking admissible deletions. Specifically, $h_t(S)$ is instantiated by a verifier score, answer confidence, or margin-based score when available; exact answer preservation is used only as the sufficiency constraint. This avoids the degenerate case in which every admissible deletion has zero answer-flip harm.

We normalize positive necessity scores as
\begin{equation}
    w_t=
    \frac{\max(\Delta_t,0)}
    {\sum_{j=1}^{T}\max(\Delta_j,0)+\eta},
    \qquad
    \mathrm{NMass}_k(z)=
    \sum_{t\in\mathrm{TopK}(w,k)} w_t ,
    \label{eq:necessity-weight}
\end{equation}
where $\eta=10^{-8}$ unless otherwise specified. High top-$k$ mass indicates sparse \emph{marginal} necessity: a small set of steps accounts for most measured unique predictive support.

\customsubsection{Extracting Minimal Cores}
\label{subsec:extracting-cores}

Solving Eq.~\eqref{eq:minimal-core} exactly requires $2^T$ subset evaluations. We therefore use practical approximations based on repeated sufficiency checks.

\paragraph{Greedy backward elimination.}
Our primary method starts from the full trace and repeatedly removes a step when doing so preserves sufficiency. At each iteration, it evaluates candidate single-step deletions and removes the admissible step with smallest estimated harm. It terminates when no remaining single-step deletion preserves sufficiency. The result is not guaranteed to be globally minimal, but is locally minimal with respect to one-step deletions.

Here $h_t(S)$ denotes the estimated harm of deleting step $t$ from the current retained subset $S$. When likelihoods are available, we use the loss increase induced by deletion. In black-box settings, $h_t(S)$ is a verifier, confidence, or margin-based score used only to rank admissible deletions; the sufficiency predicate itself still determines whether a deletion is allowed. Ties are broken deterministically by the original step index.

\begin{algorithm}[t]
\caption{Greedy Backward Elimination for Minimal-Core Extraction}
\label{alg:greedy-core}
\begin{algorithmic}[1]
\REQUIRE Input $x$; trace $z=(r_1,\ldots,r_T)$; model $M$; predicate $\mathsf{Suff}$; harm score $h_t(S)$.
\ENSURE Irreducible sufficient subset $S$.
\STATE $S \leftarrow [T]$ \hfill $\triangleright$ Start from full trace.
\STATE $\mathcal{A}(S) \leftarrow \{t\in S:\mathsf{Suff}(S\setminus\{t\};x,z,M)=1\}$ \hfill $\triangleright$ Admissible deletions.
\WHILE{$\mathcal{A}(S)\neq\emptyset$}
    \STATE $t^\dagger \in \arg\min_{t \in A(S)} \big(h_t(S), t\big)$ \hfill $\triangleright$ Least harmful; deterministic tie-break.
    \STATE $S \leftarrow S\setminus\{t^\dagger\}$ \hfill $\triangleright$ Delete one sufficient-preserving step.
    \STATE $\mathcal{A}(S) \leftarrow \{t\in S:\mathsf{Suff}(S\setminus\{t\};x,z,M)=1\}$ \hfill $\triangleright$ Recompute after interaction changes.
\ENDWHILE
\RETURN $S$ \hfill $\triangleright$ Locally irreducible sufficient subset.
\end{algorithmic}
\end{algorithm}

\vspace{-6pt}
\paragraph{Necessity-guided pruning.}
As a cheaper alternative, we estimate $\Delta_t$ once using Eq.~\eqref{eq:necessity}, rank steps by increasing necessity, and remove low-necessity steps first. This is less adaptive because it does not recompute interactions after each deletion, but it scales well for large benchmark sweeps.

\vspace{-6pt}
\paragraph{Random deletion baseline.}
We compare against random deletion at matched removal rates. Since long traces may tolerate some deletion by chance, a meaningful minimal-core procedure should preserve answers substantially better than random deletion at the same compression level.

\vspace{-6pt}
\paragraph{Complexity.}
For trace length $T$, exhaustive search requires $O(2^T)$ sufficiency checks. Greedy backward elimination requires at most $O(T^2)$ checks, because each round considers all remaining steps and at most $T$ deletions are accepted. Necessity-guided pruning requires $O(T)$ initial deletion evaluations plus sufficiency checks along the pruning path. In practice, greedy elimination is feasible for typical chain-of-thought traces, while necessity-guided pruning is useful for longer traces and larger sweeps.

Together, greedy deletion, necessity-guided pruning, and random deletion separate three effects: adaptive sufficient-core extraction, scalable necessity-based approximation, and the baseline compressibility expected from trace length alone.

\section{Formal Guarantees and Representation Geometry}
\label{sec:theory-geometry}
\vspace{-2pt}

We support minimal-core analysis in two ways: by proving sufficiency and irreducibility guarantees for extraction, and by testing whether minimal cores form more compact, correctness-informative representations than full traces or removed steps.

\customsubsection{Existence, Greedy Extraction, and Certificates}
\label{subsec:theory-certificates}

Let $z=(r_1,\ldots,r_T)$ be a finite trace, $[T]=\{1,\ldots,T\}$, and $z_S$ the subsequence indexed by $S\subseteq[T]$ in original order. Let $\mathsf{Suff}(S;x,z,M)\in\{0,1\}$ denote an answer- or distribution-preserving sufficiency predicate for model $M$ on input $x$. We write $S^\star$ for a minimum-cardinality sufficient subset, when it exists, and $z^\star=z_{S^\star}$ for its trace.

\begin{theorem}{Existence of minimal cores}{existence}
If the full trace is sufficient for itself, $\mathsf{Suff}([T];x,z,M)=1$, then there exists at least one minimal core $S^\star\subseteq[T]$.
\end{theorem}

\noindent\textit{Proof sketch.}
The sufficient family $\mathcal{F}=\{S\subseteq[T]:\mathsf{Suff}(S;x,z,M)=1\}$ is nonempty since $[T]\in\mathcal{F}$ and finite since $[T]$ is finite. Hence some $S^\star\in\mathcal{F}$ attains minimum cardinality. Full proof in Appendix~\ref{app:proofs}. \hfill$\square$

\paragraph{Irreducibility.}
A sufficient subset $S$ is \emph{irreducible} if no retained step can be removed while preserving sufficiency:
\begin{equation}
    \forall t\in S,\qquad
    \mathsf{Suff}(S\setminus\{t\};x,z,M)=0 .
\end{equation}
Every globally minimal core is irreducible, but the converse need not hold: a locally irreducible subset may still be larger than another sufficient subset reachable only through a multi-step change.

\begin{theorem}{Local irreducibility of greedy deletion}{local-minimality}
Assume the full trace is sufficient and greedy backward elimination deletes a step only when the resulting subset remains sufficient. Then the algorithm terminates and returns an irreducible sufficient subset $\widehat{S}$.
\end{theorem}

\noindent\textit{Proof sketch.}
Each accepted deletion preserves sufficiency and strictly decreases $|S|$, so the algorithm terminates after at most $T$ accepted deletions. At termination, no remaining step can be deleted while preserving sufficiency; otherwise it would be admissible. Full proof in Appendix~\ref{app:proofs}. \hfill$\square$

\paragraph{Overcompleteness certificates.}
Observed sufficient deletions certify overcompleteness. If some $t\in[T]$ satisfies $\mathsf{Suff}([T]\setminus\{t\};x,z,M)=1$, then $z$ is overcomplete. More generally, let $K\subseteq[T]$ be a deletion set with $|K|=k$. If $\mathsf{Suff}([T]\setminus K;x,z,M)=1$, then any minimal core satisfies
\[
    \mathrm{CR}(z)=\frac{|S^\star|}{T}\leq \frac{T-k}{T},
    \qquad
    \mathrm{RM}(z)=1-\mathrm{CR}(z)\geq\frac{k}{T}.
\]
Thus, every sufficient deletion set upper-bounds compression ratio and lower-bounds removable trace mass.

\paragraph{Sparse marginal necessity certificate.}
Recall the leave-one-out necessity score and normalized weight
\[
\Delta_t(z)=\ell(x,z_{[T]\setminus\{t\}})-\ell(x,z_{[T]}),
\qquad
w_t=\frac{\max(\Delta_t,0)}
{\sum_{j=1}^{T}\max(\Delta_j,0)+\eta},
\]
where $\ell(x,z_S)$ is a task loss, e.g., the negative log-likelihood of the full-trace answer under retained trace $z_S$, and $\eta>0$ is a smoothing constant. If $C\subseteq[T]$ satisfies $\sum_{t\in C}w_t\geq1-\gamma$ for $\gamma\in[0,1]$, then
\[
    \sum_{t\notin C}w_t
    =
    \sum_{t=1}^{T}w_t-\sum_{t\in C}w_t
    \leq \gamma .
\]
When $|C|\ll T$ and $\gamma$ is small, the trace has sparse \emph{marginal} necessity support, motivating top-$k$ necessity mass.

\begin{takeawaybox}
\noindent\textbf{Takeaway.}
Minimal cores exist for finite sufficient traces; greedy deletion gives a locally irreducible sufficient subset, and sufficient deletions certify overcompleteness while concentrated necessity certifies sparse predictive support.
\end{takeawaybox}

\customsubsection{Trace, Core, and Removed-Step Geometry}
\label{subsec:trace-core-geometry}

The subset results characterize which steps are sufficient. We next study representation geometry: the organization of trace embeddings in terms of compactness, correct--incorrect separability, intrinsic dimensionality, and alignment among full traces, cores, and removed steps. Let $\phi(r_t)\in\mathbb{R}^d$ be a $d$-dimensional step embedding, computed in experiments by mean-pooling token hidden states from a fixed decoder layer over step $r_t$; unless stated otherwise, we use the final transformer layer of the open-weight generator. For API-only models, we use a sentence encoder. Define the full-trace embedding $h(z)$, minimal-core embedding $h(z^\star)$, and removed-step embedding $h(z^-)$ as
\begin{equation}
\begin{aligned}
    h(z) &=
    \frac{1}{T}\sum_{t=1}^{T}\phi(r_t), \\
    h(z^\star) &=
    \begin{cases}
    |S^\star|^{-1}\sum_{t\in S^\star}\phi(r_t), & |S^\star|>0,\\
    h_{\emptyset}, & |S^\star|=0,
    \end{cases}\\
    h(z^{-}) &=
    \begin{cases}
    (T-|S^\star|)^{-1}\sum_{t\notin S^\star}\phi(r_t), & T>|S^\star|,\\
    h_{\emptyset}, & T=|S^\star|,
    \end{cases}
\end{aligned}
\label{eq:trace-core-removed-embeddings}
\end{equation}
where $z^- = z_{[T]\setminus S^\star}$ and $h_{\emptyset}$ is a fixed zero vector used only for degenerate empty-set cases and excluded from cosine-similarity analyses.

We compare full traces, minimal cores, and removed steps along three axes. First, \emph{compactness}: for $N$ embeddings $\mathcal{H}=\{h_i\}_{i=1}^N$,
\begin{equation}
    \mathrm{Var}(\mathcal{H})
    =
    \frac{1}{N}\sum_{i=1}^{N}\|h_i-\mu_{\mathcal{H}}\|_2^2,
    \qquad
    \mu_{\mathcal{H}}=\frac{1}{N}\sum_{i=1}^{N}h_i.
\end{equation}
If cores remove redundant variation, we expect $\mathrm{Var}(\{h(z_i^\star)\})<\mathrm{Var}(\{h(z_i)\})$, where $z_i$ indexes evaluation traces. Second, \emph{correctness separation}: we test whether embeddings separate correct and incorrect predictions using linear probes \cite{alain2016understanding}, nearest-neighbor classification, silhouette score \cite{rousseeuw1987silhouettes}, and Davies--Bouldin index \cite{davies1979cluster}. Third, \emph{intrinsic dimensionality \cite{aghajanyan2021intrinsic}}: we use NN estimators or PCA participation ratios to test whether cores lie closer to a lower-dimensional reasoning manifold.

\customsubsection{Necessity-Weighted Geometry}
\label{subsec:necessity-weighted-geometry}

Minimal cores give a discrete retained subset, while necessity scores give a softer continuous view. Using Eq.~\eqref{eq:necessity-weight}, define the necessity-weighted embedding
\begin{equation}
    h_{\Delta}(z)=\sum_{t=1}^{T}w_t\phi(r_t),
    \label{eq:necessity-weighted-embedding}
\end{equation}
where $\Delta$ indicates leave-one-out necessity weighting. If $\sum_{t=1}^{T}\max(\Delta_t,0)=0$, we set $h_\Delta(z)=h(z)$ and exclude the example from cosine-alignment analyses involving necessity weights; this avoids trace-dependent norm scaling and handles fully zero-necessity traces.

We test three predictions on nondegenerate examples: (i) \emph{core alignment}, where $h_\Delta(z)$ should align more closely with $h(z^\star)$ than $h(z)$ does; (ii) \emph{removed-step separation}, where $h_\Delta(z)$ should be less aligned with $h(z^-)$ than $h(z)$ is; and (iii) \emph{correctness signal}, where $h_\Delta(z)$ and $h(z^\star)$ should improve correct--incorrect separation relative to $h(z)$ under the same probe and split protocol. If these hold, minimal-core extraction does more than shorten traces: high-necessity steps form compact, correctness-informative support, while removable steps contribute diffuse variation.

\begin{takeawaybox}
\noindent\textbf{Takeaway.}
The geometry analysis tests whether minimal cores are not just shorter traces, but cleaner representations: compact, lower-dimensional, and more predictive of correctness than full traces or removed steps.
\end{takeawaybox}



\section{Experiments and Results}
\label{sec:experiments-results}
\vspace{-4pt}

\paragraph{Setup.} We evaluate four questions: whether generated traces are overcomplete, whether predictive support is concentrated in few steps, whether minimal cores transfer across model families, and whether they reveal cleaner geometry than full traces. We use six reasoning benchmarks: GSM8K, MATH500, AIME24, AMC23, GPQA-Diamond, and StrategyQA spanning arithmetic, competition mathematics, expert scientific reasoning, and commonsense multi-hop QA. We evaluate DeepSeek-R1-Distill-Qwen-7B \cite{deepseekai2025deepseekr1incentivizingreasoningcapability}, Qwen3-32B \cite{qwen3technicalreport}, DeepSeek-R1-Distill-Qwen-32B \cite{deepseekai2025deepseekr1incentivizingreasoningcapability}, and GPT-5 \cite{singh2025openai}, covering open-weight and closed-source reasoning model families. Open-weight models support the full metric suite, including likelihood-based necessity and representation geometry; GPT-5 is used for behaviorally observable answer-preserving compression, retention, pruning, and transfer. Unless otherwise stated, traces are generated with low-temperature CoT prompting and segmented into numbered or sentence-level steps. Answer retention compares the pruned-trace answer to the model's full-trace answer. Full implementation details are in Appendix~\ref{app:implementation}.

\customsubsection{Overcompleteness of Reasoning Traces}
\label{subsec:overcomplete-results}

Tab.~\ref{tab:main-results} reports the main overcompleteness results. Across datasets and models, greedy minimal-core extraction removes substantial trace mass while preserving the model's full-trace answer. On average, the minimal core contains $53\%$ of the full trace, corresponding to $46\%$ removable mass, with $86\%$ answer retention. Overcompleteness is strongest on GSM8K and StrategyQA, where traces often contain explanatory slack, and weaker on AIME24 and GPQA-Diamond, where derivations are more brittle.

\begin{table*}[t]
\centering
\tiny
\renewcommand\arraystretch{0.8}
\caption{\textbf{Main overcompleteness results}. Minimal cores remove substantial trace mass while preserving the full-trace answer. Bold indicates the best value within each dataset and underline indicates the second best.}
\label{tab:main-results}
\resizebox{\textwidth}{!}{
\begin{tabular}{llcccccc}
\toprule
\textbf{Dataset} & \textbf{Model} &
\textbf{Full Len.} &
\textbf{Core Len.} &
\corecell{$\mathrm{CR}\downarrow$} &
\corecell{$\mathrm{RM}\uparrow$} &
\corecell{Top-3 Mass $\uparrow$} &
\corecell{Retention $\uparrow$} \\
\midrule
\multirow{4}{*}{GSM8K}
& DeepSeek-R1-Distill-Qwen-7B & 8.9 & 4.4 & \corecell{0.49} & \corecell{0.51} & \corecell{0.70} & \corecell{0.90} \\
& Qwen3-32B & 8.7 & 4.1 & \corecell{0.47} & \corecell{0.53} & \corecell{0.72} & \corecell{0.93} \\
& DeepSeek-R1-Distill-Qwen-32B & 9.2 & 4.0 & \corecell{\secondbest{0.43}} & \corecell{\secondbest{0.57}} & \corecell{\secondbest{0.75}} & \corecell{\secondbest{0.94}} \\
& GPT-5 & 9.8 & 3.8 & \corecell{\textbf{0.39}} & \corecell{\textbf{0.61}} & \corecell{\textbf{0.80}} & \corecell{\textbf{0.96}} \\
\midrule
\multirow{4}{*}{MATH500}
& DeepSeek-R1-Distill-Qwen-7B & 12.3 & 7.3 & \corecell{0.59} & \corecell{0.41} & \corecell{0.59} & \corecell{0.84} \\
& Qwen3-32B & 12.8 & 7.1 & \corecell{0.55} & \corecell{0.45} & \corecell{0.63} & \corecell{0.87} \\
& DeepSeek-R1-Distill-Qwen-32B & 13.2 & 7.0 & \corecell{\secondbest{0.53}} & \corecell{\secondbest{0.47}} & \corecell{\secondbest{0.64}} & \corecell{\secondbest{0.88}} \\
& GPT-5 & 13.7 & 6.5 & \corecell{\textbf{0.47}} & \corecell{\textbf{0.53}} & \corecell{\textbf{0.69}} & \corecell{\textbf{0.90}} \\
\midrule
\multirow{4}{*}{AIME24}
& DeepSeek-R1-Distill-Qwen-7B & 10.2 & 6.7 & \corecell{0.66} & \corecell{0.34} & \corecell{0.53} & \corecell{0.77} \\
& Qwen3-32B & 10.5 & 6.5 & \corecell{0.62} & \corecell{0.38} & \corecell{0.56} & \corecell{0.81} \\
& DeepSeek-R1-Distill-Qwen-32B & 10.9 & 6.3 & \corecell{\secondbest{0.58}} & \corecell{\secondbest{0.42}} & \corecell{\secondbest{0.59}} & \corecell{\secondbest{0.83}} \\
& GPT-5 & 11.4 & 6.0 & \corecell{\textbf{0.53}} & \corecell{\textbf{0.47}} & \corecell{\textbf{0.64}} & \corecell{\textbf{0.85}} \\
\midrule
\multirow{4}{*}{AMC23}
& DeepSeek-R1-Distill-Qwen-7B & 9.8 & 6.0 & \corecell{0.61} & \corecell{0.39} & \corecell{0.58} & \corecell{0.81} \\
& Qwen3-32B & 10.0 & 5.7 & \corecell{0.57} & \corecell{0.43} & \corecell{0.61} & \corecell{0.85} \\
& DeepSeek-R1-Distill-Qwen-32B & 10.6 & 5.5 & \corecell{\secondbest{0.52}} & \corecell{\secondbest{0.48}} & \corecell{\secondbest{0.65}} & \corecell{\secondbest{0.87}} \\
& GPT-5 & 10.9 & 5.2 & \corecell{\textbf{0.48}} & \corecell{\textbf{0.52}} & \corecell{\textbf{0.69}} & \corecell{\textbf{0.90}} \\
\midrule
\multirow{4}{*}{GPQA-D}
& DeepSeek-R1-Distill-Qwen-7B & 7.8 & 5.0 & \corecell{0.64} & \corecell{0.36} & \corecell{0.55} & \corecell{0.76} \\
& Qwen3-32B & 8.1 & 4.8 & \corecell{0.59} & \corecell{0.41} & \corecell{0.60} & \corecell{0.81} \\
& DeepSeek-R1-Distill-Qwen-32B & 8.3 & 4.7 & \corecell{\secondbest{0.57}} & \corecell{\secondbest{0.43}} & \corecell{\secondbest{0.62}} & \corecell{\secondbest{0.82}} \\
& GPT-5 & 8.7 & 4.4 & \corecell{\textbf{0.51}} & \corecell{\textbf{0.49}} & \corecell{\textbf{0.66}} & \corecell{\textbf{0.84}} \\
\midrule
\multirow{4}{*}{StrategyQA}
& DeepSeek-R1-Distill-Qwen-7B & 7.4 & 3.7 & \corecell{0.50} & \corecell{0.50} & \corecell{0.69} & \corecell{0.88} \\
& Qwen3-32B & 7.5 & 3.5 & \corecell{0.47} & \corecell{0.53} & \corecell{0.73} & \corecell{0.91} \\
& DeepSeek-R1-Distill-Qwen-32B & 7.8 & 3.4 & \corecell{\secondbest{0.44}} & \corecell{\secondbest{0.56}} & \corecell{\secondbest{0.76}} & \corecell{\secondbest{0.93}} \\
& GPT-5 & 8.1 & 3.2 & \corecell{\textbf{0.40}} & \corecell{\textbf{0.60}} & \corecell{\textbf{0.80}} & \corecell{\textbf{0.95}} \\
\midrule
\multicolumn{2}{l}{\textbf{Average}} 
& 9.9$\stdgreen{1.9}$
& 5.3$\stdgreen{1.3}$
& \corecellbf{0.54}$\stdgreen{.03}$
& \corecellbf{0.46}$\stdgreen{.07}$
& \corecellbf{0.65}$\stdgreen{.09}$
& \corecellbf{0.86}$\stdgreen{.01}$ \\
\bottomrule
\end{tabular}}
\end{table*}

\vspace{-5pt}
\customsubsection{Pruning baselines}
Tab.~\ref{tab:pruning} shows that, at matched removal budgets, greedy deletion preserves the model's full-trace answer better than random or one-shot necessity-guided pruning. This indicates that removable mass is structured rather than a byproduct of trace length; we use the closest subset along the greedy path and never force deletions that violate sufficiency.

Fig.~\ref{fig:overcomplete-diagnostics} complements the tables with trend-level diagnostics: it shows the full retention--compression curve, redundancy and necessity concentration, sublinear core-length growth, and the geometry trend summarized in Tab.~\ref{tab:geometry-and-ablations}(a).

\begin{table}[t]
\centering
\tiny
\renewcommand\arraystretch{0.8}
\caption{\textbf{Answer retention} under matched pruning budgets, averaged across evaluated models. Each dataset entry reports mean $\pm$ std relative to the model's full-trace answer. ``Nec.'': one-shot necessity-guided pruning. These fixed-budget results are separate from the final locally irreducible cores in Tab.~\ref{tab:main-results}.}
\label{tab:pruning}
\resizebox{\columnwidth}{!}{
\begin{tabular}{lccccccc}
\toprule
\textbf{Method} &
\textbf{GSM8K} &
\textbf{MATH500} &
\textbf{GPQA-D} &
\textbf{StrategyQA} &
\corecell{\textbf{Avg.}} &
\corecell{$\boldsymbol{\Delta_{\mathrm{Rand}}}$} &
\corecell{$\boldsymbol{\Delta_{\mathrm{Nec}}}$} \\
\midrule
Random, 40\%
& $.68\stdgreen{.04}$ & $.49\stdgreen{.05}$ & $.45\stdgreen{.04}$ & $.64\stdgreen{.04}$
& \corecell{.57} & \corecell{--} & \corecell{--} \\
Nec.-guided, 40\%
& $.86\stdgreen{.03}$ & $.72\stdgreen{.04}$ & $.67\stdgreen{.04}$ & $.83\stdgreen{.03}$
& \corecell{.77} & \corecell{+.20} & \corecell{--} \\
\corecell{Greedy, 40\%}
& \corecell{$.92\stdgreen{.02}$} & \corecell{$.81\stdgreen{.03}$} & \corecell{$.75\stdgreen{.03}$} & \corecell{$.89\stdgreen{.02}$}
& \corecellbf{.84} & \corecellbf{+.28} & \corecellbf{+.07} \\
\midrule
Random, 50\%
& $.59\stdgreen{.05}$ & $.38\stdgreen{.06}$ & $.35\stdgreen{.05}$ & $.55\stdgreen{.05}$
& \corecell{.47} & \corecell{--} & \corecell{--} \\
Nec.-guided, 50\%
& $.80\stdgreen{.04}$ & $.63\stdgreen{.05}$ & $.58\stdgreen{.05}$ & $.77\stdgreen{.04}$
& \corecell{.70} & \corecell{+.23} & \corecell{--} \\
\corecell{Greedy, 50\%}
& \corecell{$.88\stdgreen{.03}$} & \corecell{$.74\stdgreen{.04}$} & \corecell{$.68\stdgreen{.04}$} & \corecell{$.84\stdgreen{.03}$}
& \corecellbf{.79} & \corecellbf{+.32} & \corecellbf{+.09} \\
\bottomrule
\end{tabular}}
\vspace{-5pt}
\end{table}

\begin{figure*}[t]
\centering

\begin{subfigure}{0.32\textwidth}
\centering
\begin{tikzpicture}
\begin{axis}[
    densegrid,
    width=\linewidth,
    height=0.70\linewidth,
    xlabel={Steps removed (\%)},
    ylabel={Answer retention},
    xmin=0, xmax=70,
    ymin=0.22, ymax=1.04,
    xtick={0,20,40,60},
    ytick={0.25,0.50,0.75,1.00},
]
\addplot+[greedyplot] coordinates {
(0,1.00) (10,0.98) (20,0.96) (30,0.92) (40,0.84) (50,0.79) (60,0.68) (70,0.56)
};
\addplot+[necessityplot] coordinates {
(0,1.00) (10,0.96) (20,0.92) (30,0.86) (40,0.77) (50,0.70) (60,0.56) (70,0.43)
};
\addplot+[randomplot] coordinates {
(0,1.00) (10,0.91) (20,0.80) (30,0.65) (40,0.57) (50,0.47) (60,0.33) (70,0.25)
};
\end{axis}
\end{tikzpicture}

\vspace{-0.4em}
{\scriptsize
\textcolor{coreblueDark}{\rule{0.9em}{0.45pt} Greedy}
\quad
\textcolor{coreorange}{\rule{0.9em}{0.45pt} Nec.}
\quad
\textcolor{black!55}{\rule{0.9em}{0.45pt} Rand.}
}
\vspace{-0.3em}

\caption{Retention--compression curve}
\end{subfigure}
\hfill
\begin{subfigure}{0.32\textwidth}
\centering
\begin{tikzpicture}
\begin{axis}[
    densegrid,
    width=\linewidth,
    height=0.70\linewidth,
    xlabel={Redundancy mass},
    ylabel={Density},
    xmin=0, xmax=0.95,
    ymin=0, ymax=3.4,
    xtick={0,0.25,0.5,0.75},
    ytick={0,1,2,3},
]
\addplot[smooth, very thick, color=coreblueDark] coordinates {
(0.02,0.05) (0.10,0.30) (0.18,0.80) (0.26,1.40) (0.34,2.10) (0.42,2.75)
(0.50,3.05) (0.58,2.70) (0.66,1.95) (0.74,1.10) (0.82,0.45) (0.90,0.10)
};
\addplot[smooth, thick, dashed, color=coreorange] coordinates {
(0.02,0.02) (0.12,0.15) (0.22,0.55) (0.32,1.20) (0.42,2.25) (0.52,2.95)
(0.62,2.45) (0.72,1.45) (0.82,0.55) (0.90,0.12)
};
\addplot[smooth, thick, dotted, color=coregreen] coordinates {
(0.02,0.05) (0.10,0.40) (0.18,1.05) (0.26,1.80) (0.34,2.60) (0.44,3.05)
(0.54,2.60) (0.64,1.65) (0.74,0.72) (0.84,0.20)
};
\addplot[coreblueDark, thick, dashed] coordinates {(0.47,0) (0.47,2.25)};
\node[
    font=\tiny,
    anchor=west,
    text=coreblueDark,
    fill=white,
    fill opacity=0.92,
    text opacity=1,
    inner sep=1pt
] at (axis cs:0.50,2.25) {$\mu=0.47$};
\end{axis}
\end{tikzpicture}

\vspace{-0.4em}
{\scriptsize
\textcolor{coreblueDark}{\rule{0.9em}{0.45pt} All}
\quad
\textcolor{coreorange}{\rule{0.9em}{0.45pt} Math}
\quad
\textcolor{coregreen}{\rule{0.9em}{0.45pt} Non-math}
}
\vspace{-0.3em}

\caption{Redundancy distribution}
\end{subfigure}
\hfill
\begin{subfigure}{0.32\textwidth}
\centering
\begin{tikzpicture}
\begin{axis}[
    densegrid,
    width=\linewidth,
    height=0.70\linewidth,
    xlabel={Top-$k$ steps},
    ylabel={Cumulative necessity},
    xmin=1, xmax=8,
    ymin=0, ymax=1.03,
    xtick={1,2,3,4,5,6,7,8},
    ytick={0,0.25,0.5,0.75,1.0},
]
\addplot[coreblueDark, very thick, mark=*, mark size=1.5pt] coordinates {
(1,0.29) (2,0.50) (3,0.66) (4,0.75) (5,0.82) (6,0.88) (7,0.93) (8,0.97)
};
\addplot[black!45, dashed, thick] coordinates {
(1,0.125) (2,0.25) (3,0.375) (4,0.50) (5,0.625) (6,0.75) (7,0.875) (8,1.00)
};
\node[
    font=\tiny,
    anchor=west,
    text=coreblueDark,
    fill=white,
    fill opacity=0.92,
    text opacity=1,
    inner sep=1pt
] at (axis cs:3.25,0.60) {Top-3: 66\%};
\end{axis}
\end{tikzpicture}

\vspace{-0.4em}
{\scriptsize
\textcolor{coreblueDark}{\rule{0.9em}{0.45pt} Observed}
\quad
\textcolor{black!55}{\rule{0.9em}{0.45pt} Uniform}
}
\vspace{-0.3em}

\caption{Necessity concentration}
\end{subfigure}

\vspace{0.45em}

\begin{subfigure}{0.48\textwidth}
\centering
\begin{tikzpicture}
\begin{axis}[
    densegrid,
    width=\linewidth,
    height=0.48\linewidth,
    xlabel={Full trace length},
    ylabel={Minimal core length},
    xmin=4, xmax=16,
    ymin=2, ymax=10,
    xtick={4,8,12,16},
    ytick={2,4,6,8,10},
]
\addplot[
    only marks,
    color=coreblueDark,
    mark=*,
    mark size=1.5pt,
    opacity=0.80
] coordinates {
(6,3.0) (7,3.2) (8,3.7) (9,4.1) (10,4.5) (11,5.2) (12,5.8) (13,6.1)
(14,6.7) (15,7.2) (8,4.0) (9,4.3) (10,5.0) (12,6.2) (14,7.0)
};
\addplot[color=coreorange, very thick] coordinates {(4,2.2) (16,7.8)};
\addplot[black!40, dashed, thick] coordinates {(4,4) (10,10)};
\node[
    font=\tiny,
    anchor=west,
    text=coreblueDark,
    fill=white,
    fill opacity=0.92,
    text opacity=1,
    inner sep=1pt
] at (axis cs:10.3,3.15) {sublinear growth};
\end{axis}
\end{tikzpicture}

\vspace{-0.4em}
{\scriptsize
\textcolor{coreblueDark}{$\bullet$ Examples}
\quad
\textcolor{coreorange}{\rule{0.9em}{0.45pt} Fit}
\quad
\textcolor{black!55}{\rule{0.9em}{0.45pt} $y=x$}
}
\vspace{-0.3em}

\caption{Core length vs. full length}
\end{subfigure}
\hfill
\begin{subfigure}{0.48\textwidth}
\centering
\begin{tikzpicture}
\begin{axis}[
    densegrid,
    ybar,
    width=\linewidth,
    height=0.48\linewidth,
    bar width=9pt,
    xlabel={Representation metric},
    ylabel={Relative value},
    symbolic x coords={Probe, Dim., Var.},
    xtick=data,
    ymin=0,
    ymax=1.30,
    enlarge x limits=0.22,
]
\addplot[fill=black!25, draw=black!45] coordinates {(Probe,0.67) (Dim.,1.00) (Var.,1.00)};
\addplot[fill=coreblueDark!65, draw=coreblueDark] coordinates {(Probe,0.78) (Dim.,0.66) (Var.,0.69)};
\addplot[fill=coreorange!55, draw=coreorange] coordinates {(Probe,0.55) (Dim.,1.15) (Var.,1.20)};
\end{axis}
\end{tikzpicture}

\vspace{-0.4em}
{\scriptsize
\textcolor{black!55}{\rule{0.9em}{0.8em} Full}
\quad
\textcolor{coreblueDark}{\rule{0.9em}{0.8em} Core}
\quad
\textcolor{coreorange}{\rule{0.9em}{0.8em} Removed}
}
\vspace{-0.3em}

\caption{Geometry summary}
\end{subfigure}

\caption{
\textbf{Compact diagnostics for overcomplete reasoning traces.}
The panels show: (a) greedy deletion degrades more gracefully than necessity-guided or random pruning beyond the matched budgets in Tab.~\ref{tab:pruning}; (b) redundancy mass concentrates at moderate-to-high values for both math and non-math tasks; (c) necessity mass is concentrated in a few steps, above a uniform baseline; (d) minimal-core length grows sublinearly with full-trace length; and (e) minimal cores improve correctness separation while reducing intrinsic dimensionality and variance.
}
\label{fig:overcomplete-diagnostics}
\vspace{-12pt}
\end{figure*}
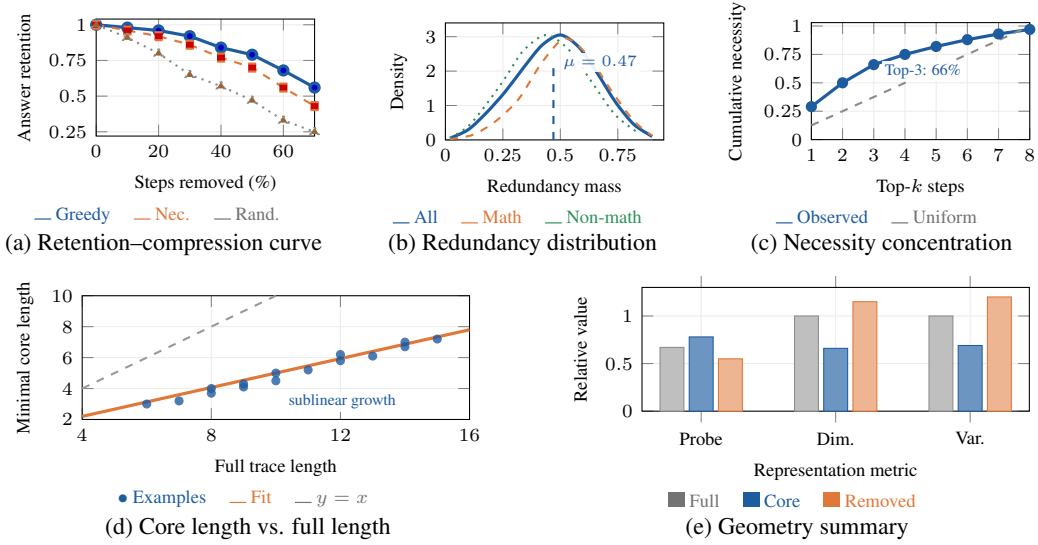

\vspace{-5pt}


\customsubsection{Necessity Concentration and Transfer}
\label{subsec:necessity-transfer-geometry}

Tab.~\ref{tab:compact-support-transfer-geometry} consolidates three compact analyses. Panel (a) shows that predictive support is sparse: the top three steps account for $65$--$66\%$ of measured necessity mass on average, and the top five account for $82\%$. Panel (b) shows that minimal cores transfer across model families with high off-diagonal retention while using nearly half the length of full traces; random matched-length subsets transfer much worse. Panel (c) shows that necessity-weighted embeddings approximate minimal-core geometry without hard subset selection, improving probe accuracy and reducing intrinsic dimensionality relative to full-trace averaging.

\begin{table*}[t]
\centering
\small
\caption{\textbf{Compact support, transfer, and necessity-weighted geometry results.}
(a) Necessity mass is concentrated in a few steps. (b) Minimal cores preserve most transferable signal while using much shorter traces. (c) Necessity-weighted embeddings approximate minimal-core geometry.}
\label{tab:compact-support-transfer-geometry}

\begin{minipage}[t]{0.32\textwidth}
\centering
\scriptsize
\textbf{(a) Necessity concentration}\\[0.2em]
\begin{tabular}{lcccc}
\toprule
\textbf{Dataset} & T1 & \corecell{T3} & T5 & Gini \\
\midrule
GSM8K & 0.34 & \corecell{0.73} & 0.88 & 0.61 \\
MATH & 0.27 & \corecell{0.64} & 0.81 & 0.53 \\
AIME & 0.24 & \corecell{0.58} & 0.76 & 0.47 \\
AMC & 0.26 & \corecell{0.62} & 0.79 & 0.50 \\
GPQA & 0.25 & \corecell{0.59} & 0.77 & 0.48 \\
Strat. & 0.35 & \corecell{0.76} & 0.90 & 0.64 \\
\midrule
\textbf{Avg.} & 0.29 & \corecellbf{0.65} & 0.82 & 0.54 \\
\bottomrule
\end{tabular}
\end{minipage}
\hfill
\begin{minipage}[t]{0.31\textwidth}
\centering
\scriptsize
\textbf{(b) Cross-model transfer}\\[0.2em]
\begin{tabular}{lccc}
\toprule
\textbf{Condition} & Len. & Ret. & Rel. \\
\midrule
Full trace & 9.9 & \corecell{0.88} & 1.00 \\
Minimal core & 5.2 & \corecellbf{0.85} & \corecellbf{0.53} \\
Random subset & 5.2 & 0.64 & \corecell{0.53} \\
\bottomrule
\end{tabular}

\vspace{0.7em}
\begin{minipage}{0.92\linewidth}
\scriptsize
\emph{Ret.} denotes off-diagonal answer retention. Minimal cores retain most transferable signal at $53\%$ relative length.
\end{minipage}
\end{minipage}
\hfill
\begin{minipage}[t]{0.32\textwidth}
\centering
\scriptsize
\textbf{(c) Necessity-weighted geometry}\\[0.2em]
\begin{tabular}{lccc}
\toprule
\textbf{Rep.} & Cos. & Probe & Dim. \\
\midrule
Full mean & 0.71 & 0.67 & 21.3 \\
\corecell{Nec.-wtd.} & \corecell{0.84} & \corecell{0.76} & \corecell{14.6} \\
\corecellbf{Core} & \corecellbf{1.00} & \corecellbf{0.78} & \corecellbf{13.9} \\
Removed & 0.49 & 0.55 & 24.7 \\
\bottomrule
\end{tabular}

\vspace{0.7em}
\begin{minipage}{0.92\linewidth}
\scriptsize
Cos. is cosine similarity to the minimal-core embedding; Dim. is estimated intrinsic dimensionality.
\end{minipage}
\end{minipage}
\vspace{-0.4em}
\end{table*}

\customsubsection{Geometry of Minimal Cores}
\label{subsec:geometry-results}

We compare full-trace, minimal-core, and removed-step embeddings using correctness probes, cluster separation, intrinsic dimensionality, and variance. Tab.~\ref{tab:geometry-and-ablations}(a) shows that minimal cores are geometrically cleaner: relative to full traces, they improve probe accuracy by roughly $11$ points and reduce intrinsic dimensionality by roughly $34\%$. Removed-step embeddings are less predictive and more diffuse, supporting the view that discarded steps carry weaker task-relevant signal.

\begin{table*}[!t]
\centering
\scriptsize
\setlength{\tabcolsep}{2.6pt}
\renewcommand{\arraystretch}{0.82}
\caption{\textbf{Geometry and robustness analyses.}
Left: geometry metrics for full traces, minimal cores, and removed steps: Probe is correctness-probe accuracy, Sil. is silhouette, DB is Davies--Bouldin index, Dim. is intrinsic dimensionality, and Var. is embedding variance. Right: ablations, where Suff., Diff., and Seg. denote sufficiency criterion, difficulty split, and segmentation rule; R is answer retention and T3 is Top-3 necessity mass.}
\label{tab:geometry-and-ablations}

\begin{minipage}[t]{0.58\textwidth}
\centering
\textbf{(a) Geometry of full traces, cores, and removed steps}\\[0.25em]
\resizebox{\linewidth}{!}{
\begin{tabular}{llccccc}
\toprule
\textbf{Dataset} & \textbf{Representation} &
\textbf{Probe $\uparrow$} &
\textbf{Sil. $\uparrow$} &
\textbf{DB $\downarrow$} &
\textbf{Dim. $\downarrow$} &
\textbf{Var. $\downarrow$} \\
\midrule
\multirow{3}{*}{GSM8K}
& Full trace & 0.71 & 0.18 & 1.42 & 18.7 & 1.00 \\
& \corecell{\textbf{Minimal core}} & \corecell{\textbf{0.82}} & \corecell{\textbf{0.31}} & \corecell{\textbf{0.96}} & \corecell{\textbf{11.3}} & \corecell{\textbf{0.63}} \\
& Removed steps & 0.56 & 0.07 & 1.88 & 22.1 & 1.24 \\
\midrule
\multirow{3}{*}{MATH500}
& Full trace & 0.66 & 0.12 & 1.67 & 24.1 & 1.00 \\
& \corecell{\textbf{Minimal core}} & \corecell{\textbf{0.75}} & \corecell{\textbf{0.23}} & \corecell{\textbf{1.21}} & \corecell{\textbf{16.2}} & \corecell{\textbf{0.71}} \\
& Removed steps & 0.54 & 0.05 & 2.04 & 27.5 & 1.19 \\
\midrule
\multirow{3}{*}{GPQA-D}
& Full trace & 0.63 & 0.10 & 1.74 & 21.8 & 1.00 \\
& \corecell{\textbf{Minimal core}} & \corecell{\textbf{0.72}} & \corecell{\textbf{0.20}} & \corecell{\textbf{1.28}} & \corecell{\textbf{15.0}} & \corecell{\textbf{0.74}} \\
& Removed steps & 0.53 & 0.04 & 2.12 & 25.7 & 1.17 \\
\midrule
\multirow{3}{*}{StrategyQA}
& Full trace & 0.69 & 0.16 & 1.51 & 20.5 & 1.00 \\
& \corecell{\textbf{Minimal core}} & \corecell{\textbf{0.80}} & \corecell{\textbf{0.29}} & \corecell{\textbf{1.02}} & \corecell{\textbf{13.4}} & \corecell{\textbf{0.67}} \\
& Removed steps & 0.55 & 0.06 & 1.94 & 23.6 & 1.22 \\
\bottomrule
\end{tabular}}
\end{minipage}
\hfill
\begin{minipage}[t]{0.39\textwidth}
\centering
\textbf{(b) Ablations and sensitivity}\\[0.25em]
\tiny
\resizebox{\linewidth}{!}{
\begin{tabular}{llcccc}
\toprule
\textbf{Grp.} & \textbf{Setting} & \textbf{Len.} & \textbf{CR$\downarrow$} & \textbf{RM$\uparrow$} & \textbf{Ret./T3} \\
\midrule
\multirow{3}{*}{Suff.}
& Ans. & -- & \corecell{.54} & \corecell{.46} & .88 R \\
& Dist. .10 & -- & .61 & .39 & .91 R \\
& Dist. .05 & -- & .66 & .34 & .93 R \\
\midrule
\multirow{3}{*}{Prompt}
& Std. & 9.8 & \corecell{.54} & \corecell{.46} & \corecell{.72 T3} \\
& Plan & 10.6 & .57 & .43 & .69 T3 \\
& Concise & 6.4 & .68 & .32 & .61 T3 \\
\midrule
\multirow{3}{*}{Diff.}
& Easy & -- & \corecell{.42} & \corecell{.58} & \corecell{.91 R} \\
& Med. & -- & .54 & .46 & .82 R \\
& Hard & -- & .66 & .34 & .69 R \\
\midrule
\multirow{3}{*}{Seg.}
& Sent. & 12.8 & .49 & .51 & .86 R \\
& Num. & 9.7 & \corecell{.54} & \corecell{.46} & \corecell{.88 R} \\
& Para. & 5.1 & .67 & .33 & .90 R \\
\bottomrule
\end{tabular}}
\end{minipage}

\vspace{-0.7em}
\end{table*}

\customsubsection{Ablations and Sensitivity}
\label{subsec:ablations-sensitivity}

Tab.~\ref{tab:geometry-and-ablations}(b) summarizes the main robustness checks. Distribution-preserving sufficiency is stricter and yields larger cores, but traces remain compressible. Plan-then-solve traces are more organized but still overcomplete; concise-CoT reduces trace length but also reduces removable mass. Difficulty stratification shows that harder problems are less compressible and more sensitive to aggressive pruning. Segmentation changes absolute core length but preserves the qualitative trend.

\vspace{-6pt}
\section{Related Works}
\vspace{-6pt}

\paragraph{Efficient and minimal reasoning.}
Chain-of-thought prompting improves complex reasoning~\citep{wei2022chain,kojima2022large,wang2023selfconsistency}, with extensions via decomposition~\citep{radhakrishnan2023question,yixing2024chain}, search~\citep{li2025chain,yao2023tree}, verification~\citep{zhao2023verify,ling2023deductive}, and process supervision~\citep{zhou2023least,lightman2023lets,zelikman2022star}. Recent work studies whether long traces are necessary, using distillation, test-time scaling, adaptive computation, RL, trimming, or step/token skipping~\citep{ye2025limo,muennighoff2025s1,snell2024scaling,wang2026trims,chen2025surgical,li2026making}. Closest to ours, \citet{li2026making} prune low-entropy CoT steps and train models to emit compressed traces with \texttt{[SKIP]} tokens. Our work is complementary: instead of using entropy as a compression heuristic or training shorter traces, we identify minimal answer- or distribution-preserving cores in already generated traces, enabling post-hoc measurements of redundancy, necessity, transfer, geometry, and overcompleteness.

\vspace{-8pt}
\paragraph{Representation geometry.}
Neural representation geometry \cite{li2025tracing, sun2026llm} has been studied using probes, similarity analysis, clustering, interventions, and intrinsic-dimensionality estimates~\citep{alain2016understanding,kornblith2019similarity,hewitt2019structural,tenney2019bertology,ethayarajh2019contextual,belinkov2022probing,aghajanyan2021intrinsic}. Recent work further shows that language-model representations encode truthfulness \cite{zhang2024truthx, joshi2024personas, liu2023cognitive}, uncertainty \cite{savage2025large, ren2023robots, liu2025uncertainty}, and latent decision states exposed through probing or intervention~\citep{li2023inference,burns2023discovering,marks2023geometry,hernandez2023linearity}. We test this through compactness, intrinsic dimensionality, and correctness predictiveness of full, core, removed-step, and necessity-weighted embeddings.

\vspace{-4pt}
\paragraph{Rationale faithfulness.}
Work on rationale faithfulness uses erasure, counterfactuals, causal mediation, and sufficiency tests to ask whether explanations support model predictions~\citep{lei2016rationalizing,jain2019attention,wiegreffe2019attention,deyoung2020eraser,jacovi2020towards,chan2022causal, amuse, egoadapt, magnet, aurelia, avtrustbench, meerkat, melfusion, apollo, aura}. We instead analyze generated reasoning steps themselves: minimal cores are behavioral sufficient supports, not claims of complete causal or human-rationale faithfulness.

\vspace{-6pt}
\section{Conclusion and Limitations}
\label{sec:conclusion}
\vspace{-6pt}

We studied reasoning traces as overcomplete objects and introduced minimal cores: the smallest answer- or distribution-preserving subsets of generated steps. Our framework quantifies behavioral necessity through compression ratio, redundancy mass, and necessity concentration. Across benchmarks and models, minimal cores preserve predictions while removing substantial trace mass, transfer better than random subsets, and yield compact, correctness-informative representations. Thus, long traces often contain removable scaffolding, while smaller cores carry most predictive support.

\vspace{-4pt}
\noindent{\textbf{Limitations.}}
Our sufficiency notion is behavioral and model-relative, so minimal cores are not complete causal or human-rationale explanations. Greedy deletion is practical but only locally irreducible, not globally optimal. Likelihood- or hidden-state-based metrics require open-weight models; closed-source models support only behavioral evaluations. Results may vary with prompting, segmentation, domain, and longer or interactive reasoning.

\clearpage
\raggedbottom
\clearpage
\bibliographystyle{unsrtnat}
\bibliography{ref}







\newpage

\appendix
\appendixpage

\startcontents[sections]
\printcontents[sections]{l}{1}{\setcounter{tocdepth}{2}}
\newpage

\section{Implementation Details}
\label{app:implementation}

This section provides the experimental details needed to reproduce the analyses in this paper. We first describe the benchmark suite in Tab.~\ref{tab:datasets-app}, then summarize model access and usage in Tab.~\ref{tab:model-summary}. We also specify the trace-generation, pruning, sufficiency, and representation-analysis protocols used throughout the appendix.

\paragraph{Compute and infrastructure.}
Open-weight experiments were run on $8\times$ NVIDIA A100 80GB GPU machine using inference-only execution; no model fine-tuning was performed. Larger models were served with tensor-parallel and batched inference where needed. The main computational cost is greedy minimal-core extraction, since each trace requires repeated sufficiency checks over candidate step deletions, with worst-case cost $O(T^2)$ for trace length $T$. In practice, because our traces are short, the full open-weight sweep over DeepSeek-R1-Distill-Qwen-7B, Qwen3-32B, and DeepSeek-R1-Distill-Qwen-32B completed in a few days on this machine. Closed-source GPT experiments used API calls for generation, answer-preserving sufficiency checks, and transfer, and took on the order of one to two days depending on rate limits. All reported experiments are inference-only, so the compute footprint is dominated by repeated pruning-time model evaluations rather than training.

\paragraph{Datasets.}
We evaluate on six deliberative reasoning benchmarks spanning arithmetic, competition mathematics, expert scientific reasoning, and commonsense multi-hop QA. GSM8K provides arithmetic word problems with reliable answer checking. MATH500 and the AIME24/AMC23 subsets test longer and more brittle mathematical derivations. GPQA-Diamond evaluates expert-level scientific QA, and StrategyQA tests non-math multi-hop commonsense reasoning. Tab.~\ref{tab:datasets-app} summarizes the role of each benchmark in the evaluation.

\begin{table}[h]
\centering
\appsmalltable
\caption{Benchmarks used in our study. The suite covers both math-heavy and non-math reasoning, allowing us to test whether overcompleteness is tied to a particular reasoning domain.}
\label{tab:datasets-app}
\begin{tabular}{p{2.2cm}p{2.5cm}p{6.0cm}}
\toprule
\textbf{Dataset} & \textbf{Reasoning type} & \textbf{Role in our evaluation} \\
\midrule
GSM8K & Arithmetic reasoning & Medium-length CoT traces with reliable answer checking. \\
MATH500 & Advanced math & Longer mathematical derivations and higher sensitivity to deletion. \\
AIME24 & Competition math & Recent hard math setting aligned with modern reasoning evaluations. \\
AMC23 & Competition math & Competition-style derivations with moderate trace length. \\
GPQA-Diamond & Expert science QA & Tests overcompleteness under expert-level non-school reasoning. \\
StrategyQA & Commonsense multi-hop & Tests generality beyond math-heavy reasoning traces. \\
\bottomrule
\end{tabular}
\end{table}

\paragraph{Models.}
We evaluate DeepSeek-R1-Distill-Qwen-7B, Qwen3-32B, DeepSeek-R1-Distill-Qwen-32B, and GPT-5. The open-weight models allow likelihood and hidden-state access for distributional sufficiency, necessity scoring, and representation analyses, while GPT-5 is used for black-box answer-preserving compression, retention, pruning baselines, and transfer. Tab.~\ref{tab:model-summary} summarizes the model families and their roles.

\paragraph{Metric Definitions and Ranges}
\label{app:metric-definitions}

This section summarizes the metrics used throughout the paper and how to interpret them. For a trace of length $T$ and extracted core $S^\star$, the \emph{compression ratio} is
\[
\mathrm{CR}(z)=\frac{|S^\star|}{T},
\]
and the \emph{redundancy mass} is
\[
\mathrm{RM}(z)=1-\mathrm{CR}(z).
\]
Both lie in $[0,1]$: lower $\mathrm{CR}$ means a shorter retained core, while higher $\mathrm{RM}$ means more removable trace mass.

Step necessity is measured by the leave-one-out loss change
\[
\Delta_t(z)=\ell(x,z_{[T]\setminus\{t\}})-\ell(x,z_{[T]}),
\]
where $\ell$ is typically the negative log-likelihood of the full-trace answer. Positive $\Delta_t$ indicates that removing step $t$ increases loss, while near-zero or negative values indicate little unique leave-one-out support. We convert positive necessity scores into normalized weights
\[
w_t=
\frac{\max(\Delta_t,0)}
{\sum_{j=1}^{T}\max(\Delta_j,0)+\eta},
\]
with $\eta=10^{-8}$. The top-$k$ necessity mass is
\[
\mathrm{NMass}_k(z)=\sum_{t\in \mathrm{TopK}(w,k)} w_t.
\]
It lies in $[0,1]$ up to numerical smoothing. We report \emph{Top-3 Mass} as $\mathrm{NMass}_3$, the fraction of positive marginal necessity assigned to the three highest-necessity steps. Higher values indicate more concentrated predictive support.

Answer retention is the fraction of examples for which a retained trace preserves the model's full-trace answer:
\[
\mathrm{Ret}
=
\frac{1}{N}\sum_{i=1}^{N}
\mathbbm{1}\{
\hat y_M(x_i,z_{S_i})=\hat y_M(x_i,z_{[T_i]})
\}.
\]
Retention also lies in $[0,1]$. It is model-relative: it measures preservation of the model's original full-trace prediction, not necessarily correctness against the dataset label. Dataset labels are used separately for accuracy and correctness-geometry analyses.

\begin{table}[t]
\centering
\small
\caption{Model families used in the experiments. Open-weight models support likelihood and hidden-state analyses; GPT-5 is used for black-box behavioral metrics.}
\label{tab:model-summary}
\begin{tabular}{llll}
\toprule
\textbf{Model} & \textbf{Size} & \textbf{Access} & \textbf{Used for} \\
\midrule
DeepSeek-R1-Distill-Qwen & 7B & Open weights & pruning, likelihood, geometry \\
Qwen3 & 32B & Open weights & pruning, likelihood, geometry \\
DeepSeek-R1-Distill-Qwen & 32B & Open weights & pruning, likelihood, geometry \\
GPT-5 & API & Closed & black-box retention, pruning, transfer \\
\bottomrule
\end{tabular}
\end{table}

\paragraph{Trace generation and segmentation.}
Unless otherwise stated, traces are generated using low-temperature chain-of-thought prompting. Models are prompted to emit numbered reasoning steps. When outputs are not explicitly numbered, we apply deterministic sentence-level segmentation and merge short continuations. The same segmentation rule is used for full traces, pruning, and representation analysis. The segmentation sensitivity study in Tab.~\ref{tab:robustness-app} and Fig.~\ref{fig:app-robustness-dense} tests whether the main conclusions depend on this choice.

\paragraph{Sufficiency and answer retention.}
For answer-preserving sufficiency, we compare the answer produced from a retained subset against the full-trace answer, not necessarily the ground-truth answer. Thus answer retention measures preservation of the model's original prediction. For distribution-preserving sufficiency, we use KL divergence over the candidate answer distribution when token likelihoods are available; otherwise, we report only answer preservation. The sufficiency-oracle ablation in Tab.~\ref{tab:oracle-ablation-app} tests whether overcompleteness remains under stricter distributional criteria.

\paragraph{Pruning-budget matching.}
Greedy deletion is adaptive and terminates when no further sufficient-preserving deletion is available. For matched-budget comparisons, we evaluate retained subsets encountered along the greedy deletion path at the closest available removal budget. If the greedy path terminates before a target budget, we report the terminal locally irreducible subset rather than forcing deletions that violate sufficiency. Random and necessity-guided baselines are evaluated at the same nominal removal budgets; their cost--fidelity tradeoff is summarized in Tab.~\ref{tab:extraction-method-app} and Fig.~\ref{fig:app-extraction-dense}.

\paragraph{Representation analysis.}
For open-weight models, step embeddings are computed by mean-pooling final-layer token hidden states over the token span corresponding to each segmented step. For API-only models, we use a sentence encoder over the textual step. Linear-probe accuracy is evaluated on fixed held-out splits, and intrinsic dimensionality is estimated using PCA participation ratio unless otherwise specified. We use $\eta=10^{-8}$ for normalized necessity weights. The additional representation results in Tab.~\ref{tab:geometry-gain-app} and Fig.~\ref{fig:app-geometry-dense} evaluate whether minimal cores yield cleaner trace embeddings.

\begin{apptakeawaybox}
\noindent\textbf{Takeaway.}
All pruning, retention, and representation analyses use a fixed segmentation and sufficiency protocol within each experiment, so observed differences reflect retained-step structure rather than preprocessing changes.
\end{apptakeawaybox}

\section{Proof Details}
\label{app:proofs}

This section provides the full proof details for the theoretical claims in the main paper. We include three basic structural claims: existence of minimal cores, irreducibility of globally minimal cores, and local irreducibility of greedy deletion. We then give two certificates corresponding to the main text: overcompleteness/compression and sparse marginal necessity.

\subsection{Existence of Minimal Cores}

The first result states that the minimal-core object is well-defined whenever the full trace itself is sufficient. This finite-feasibility argument justifies the minimal-core optimization problem used in the main paper.

\begin{appthmbox}{Existence of minimal cores}
If the full trace is sufficient for itself,
$\mathsf{Suff}([T];x,z,M)=1$, then there exists at least one minimal core
$S^\star\subseteq[T]$.
\end{appthmbox}

\begin{proof}
Define the family of sufficient subsets
\[
    \mathcal{F}
    =
    \{S\subseteq[T] : \mathsf{Suff}(S;x,z,M)=1\}.
\]
The assumption $\mathsf{Suff}([T];x,z,M)=1$ implies $[T]\in\mathcal{F}$, so $\mathcal{F}$ is nonempty. Since the trace length $T$ is finite, the power set $2^{[T]}$ is finite, and therefore $\mathcal{F}\subseteq2^{[T]}$ is finite. The set of feasible cardinalities $\{|S|:S\in\mathcal{F}\}$ is thus a nonempty finite set of nonnegative integers and must attain a minimum. Any sufficient subset $S^\star\in\mathcal{F}$ attaining this minimum satisfies
\[
    S^\star\in
    \arg\min_{S\subseteq[T]} |S|
    \quad
    \mathrm{s.t.}
    \quad
    \mathsf{Suff}(S;x,z,M)=1,
\]
and is therefore a minimal core. The argument does not require uniqueness; if multiple subsets attain the same minimum cardinality, each is a valid minimal core.
\end{proof}

\subsection{Global Minimality and Local Irreducibility}

The next two results distinguish global and local minimality. The lemma below shows that any true minimum-cardinality core must be irreducible. The following theorem shows that greedy backward deletion, although not guaranteed to find a globally minimum subset, always terminates at a locally irreducible sufficient subset.

\begin{applembox}{Global minimality implies irreducibility}
If $S^\star$ is a minimum-cardinality sufficient subset, then $S^\star$ is irreducible.
\end{applembox}

\begin{proof}
Suppose for contradiction that $S^\star$ is a minimum-cardinality sufficient subset but is not irreducible. Then there exists a retained step $t\in S^\star$ such that
\[
    \mathsf{Suff}(S^\star\setminus\{t\};x,z,M)=1.
\]
The set $S^\star\setminus\{t\}$ is sufficient and has cardinality $|S^\star|-1$, strictly smaller than $|S^\star|$. This contradicts the assumption that $S^\star$ has minimum cardinality among all sufficient subsets. Therefore no retained step can be removed individually while preserving sufficiency, and $S^\star$ is irreducible.
\end{proof}

\begin{appthmbox}{Local irreducibility of greedy deletion}
Assume the full trace is sufficient and greedy backward elimination deletes a step only when the resulting subset remains sufficient. Then the algorithm terminates and returns an irreducible sufficient subset $\widehat{S}$.
\end{appthmbox}

\begin{proof}
Let $S_0=[T]$ be the initial subset. By assumption, $\mathsf{Suff}(S_0;x,z,M)=1$. At iteration $k$, greedy deletion updates $S_{k+1}=S_k\setminus\{t_k\}$ only if
\[
    \mathsf{Suff}(S_k\setminus\{t_k\};x,z,M)=1.
\]
Therefore, by induction over iterations, every accepted subset $S_k$ remains sufficient. Each accepted deletion decreases the subset cardinality by one, $|S_{k+1}|=|S_k|-1$. Since cardinality is a nonnegative integer, the algorithm can accept at most $T$ deletions and must terminate.

Let $\widehat{S}$ be the returned subset. Sufficiency follows from the invariant above. If $\widehat{S}$ were not irreducible, then there would exist some $t\in\widehat{S}$ such that
\[
    \mathsf{Suff}(\widehat{S}\setminus\{t\};x,z,M)=1.
\]
But then $t$ would be an admissible deletion at termination, contradicting the stopping condition. Hence $\widehat{S}$ is irreducible.
\end{proof}

\subsection{Overcompleteness and Compression Certificates}

The main paper uses sufficient deletions as certificates of overcompleteness. The following proposition combines the single-step and multi-step deletion cases. The single-step case certifies that the trace is overcomplete; the multi-step case additionally gives bounds on compression ratio and removable mass.

\begin{apppropbox}{Overcompleteness and compression certificates}
If there exists $t\in[T]$ such that
\[
    \mathsf{Suff}([T]\setminus\{t\};x,z,M)=1,
\]
then $z$ is overcomplete. More generally, if there exists $K\subseteq[T]$ with $|K|=k$ such that
\[
    \mathsf{Suff}([T]\setminus K;x,z,M)=1,
\]
then any minimal core $S^\star$ satisfies
\[
    \mathrm{CR}(z)=\frac{|S^\star|}{T}\leq\frac{T-k}{T},
    \qquad
    \mathrm{RM}(z)=1-\mathrm{CR}(z)\geq\frac{k}{T}.
\]
\end{apppropbox}

\begin{proof}
For the single-step case, $[T]\setminus\{t\}$ is a proper subset of the full trace that remains sufficient. This is exactly the definition of an overcomplete trace.

For the multi-step case, the condition $\mathsf{Suff}([T]\setminus K;x,z,M)=1$ gives a sufficient subset of size $T-k$. Since $S^\star$ is a minimum-cardinality sufficient subset,
\[
    |S^\star|\leq T-k.
\]
Dividing by $T$ gives
\[
    \mathrm{CR}(z)=\frac{|S^\star|}{T}\leq\frac{T-k}{T}.
\]
Using $\mathrm{RM}(z)=1-\mathrm{CR}(z)$ gives
\[
    \mathrm{RM}(z)\geq 1-\frac{T-k}{T}=\frac{k}{T}.
\]
Thus any observed sufficient deletion of $k$ steps upper-bounds compression ratio and lower-bounds removable mass.
\end{proof}

\subsection{Sparse Marginal Necessity Certificate}

The final certificate concerns the normalized leave-one-out necessity weights used to quantify concentration of measured marginal support. It shows that if a small set $C$ accounts for almost all normalized positive necessity mass, then its complement has bounded residual mass.

\begin{apppropbox}{Sparse marginal necessity certificate}
Let
\[
w_t=
\frac{\max(\Delta_t,0)}
{\sum_{j=1}^{T}\max(\Delta_j,0)+\eta},
\qquad \eta>0.
\]
If $C\subseteq[T]$ satisfies $\sum_{t\in C}w_t\geq1-\gamma$, then
\[
    \sum_{t\notin C}w_t\leq\gamma .
\]
\end{apppropbox}

\begin{proof}
By construction, $w_t\geq0$ for all $t$, and
\[
    \sum_{t=1}^{T}w_t
    =
    \frac{\sum_{t=1}^{T}\max(\Delta_t,0)}
    {\sum_{t=1}^{T}\max(\Delta_t,0)+\eta}
    \leq 1.
\]
If $C\subseteq[T]$ satisfies $\sum_{t\in C}w_t\geq1-\gamma$, then
\[
    \sum_{t\notin C}w_t
    =
    \sum_{t=1}^{T}w_t-\sum_{t\in C}w_t
    \leq
    1-(1-\gamma)
    =
    \gamma.
\]
Therefore the complement $[T]\setminus C$ carries at most $\gamma$ normalized marginal-necessity mass. This is a sparsity certificate for measured marginal necessity, not a claim that all causally useful reasoning is contained in $C$ when steps are redundant or substitutable.
\end{proof}

\begin{apptakeawaybox}
\noindent\textbf{Takeaway.}
The theoretical claims are deliberately modest: minimal cores exist, greedy deletion returns a locally irreducible sufficient subset, and observed deletion/necessity patterns give certificates of overcompleteness and sparse marginal support.
\end{apptakeawaybox}

\section{Additional Transfer Experiments}
\label{app:transfer-experiments}

\paragraph{Dense benchmark diagnostics.}
Fig.~\ref{fig:app-overcomplete-dense} visualizes the same phenomenon from four complementary perspectives. Panel (a) compares redundancy mass and answer retention across benchmarks, with shaded bands showing cross-model variability. Panel (b) plots the model-wise redundancy profiles and shows that all model families exhibit the same broad difficulty-sensitive trend. Panel (c) compares full trace length to core length and illustrates that core length grows sublinearly rather than tracking the full trace one-to-one. Panel (d) shows that high redundancy can coexist with high answer retention, supporting the interpretation that much of the removed mass is behaviorally unnecessary under the chosen sufficiency criterion.

\begin{figure}[h]
\centering
\begin{tikzpicture}
\begin{groupplot}[
    group style={group size=2 by 2, horizontal sep=1.25cm, vertical sep=1.80cm},
    width=0.46\textwidth,
    height=4.45cm,
    grid=both,
    major grid style={draw=gray!25},
    tick label style={font=\scriptsize},
    label style={font=\small},
    title style={font=\small\bfseries},
]
\nextgroupplot[
    title={(a) Redundancy and retention by benchmark},
    ymin=0.30, ymax=1.00,
    xmin=1, xmax=6,
    xtick={1,2,3,4,5,6},
    xticklabels={GSM8K,MATH,AIME,AMC,GPQA,Strat.},
    x tick label style={rotate=25, anchor=east, font=\scriptsize, yshift=-1pt},
    ylabel={Metric value},
]
\addplot[name path=rmupper, draw=none, smooth] coordinates {
(1.0,0.61) (1.35,0.58) (1.75,0.50) (2.0,0.53)
(2.35,0.49) (2.70,0.44) (3.0,0.47) (3.35,0.50)
(3.70,0.46) (4.0,0.52) (4.35,0.48) (4.70,0.45)
(5.0,0.49) (5.35,0.54) (5.70,0.58) (6.0,0.60)
};
\addplot[name path=rmlower, draw=none, smooth] coordinates {
(1.0,0.53) (1.35,0.50) (1.75,0.44) (2.0,0.42)
(2.35,0.40) (2.70,0.37) (3.0,0.36) (3.35,0.39)
(3.70,0.38) (4.0,0.41) (4.35,0.39) (4.70,0.37)
(5.0,0.38) (5.35,0.45) (5.70,0.50) (6.0,0.53)
};
\addplot[fill=blue!18, opacity=0.70] fill between[of=rmupper and rmlower];
\addplot[blue!80!black, very thick, smooth, mark=*, mark size=1.5pt] coordinates {
(1,0.56) (2,0.47) (3,0.41) (4,0.46) (5,0.43) (6,0.56)
};

\addplot[name path=retupper, draw=none, smooth] coordinates {
(1.0,0.96) (1.4,0.94) (1.8,0.91) (2.0,0.90)
(2.5,0.88) (3.0,0.85) (3.5,0.88) (4.0,0.90)
(4.5,0.86) (5.0,0.84) (5.5,0.90) (6.0,0.95)
};
\addplot[name path=retlower, draw=none, smooth] coordinates {
(1.0,0.92) (1.4,0.90) (1.8,0.87) (2.0,0.85)
(2.5,0.82) (3.0,0.79) (3.5,0.81) (4.0,0.83)
(4.5,0.80) (5.0,0.78) (5.5,0.86) (6.0,0.90)
};
\addplot[fill=green!18, opacity=0.65] fill between[of=retupper and retlower];
\addplot[green!55!black, very thick, smooth, mark=square*, mark size=1.5pt] coordinates {
(1,0.94) (2,0.88) (3,0.82) (4,0.86) (5,0.81) (6,0.92)
};

\nextgroupplot[
    title={(b) Model-wise redundancy profile},
    ymin=0.30, ymax=0.66,
    xmin=1, xmax=6,
    xtick={1,2,3,4,5,6},
    xticklabels={GSM8K,MATH,AIME,AMC,GPQA,Strat.},
    x tick label style={rotate=25, anchor=east, font=\scriptsize, yshift=-1pt},
    ylabel={Redundancy mass},
]
\addplot[name path=modelup, draw=none, smooth] coordinates {
(1,0.61) (1.5,0.56) (2,0.53) (2.5,0.50) (3,0.47)
(3.5,0.50) (4,0.52) (4.5,0.50) (5,0.49) (5.5,0.56) (6,0.60)
};
\addplot[name path=modello, draw=none, smooth] coordinates {
(1,0.53) (1.5,0.47) (2,0.42) (2.5,0.39) (3,0.36)
(3.5,0.39) (4,0.41) (4.5,0.40) (5,0.38) (5.5,0.46) (6,0.53)
};
\addplot[fill=purple!15, opacity=0.55] fill between[of=modelup and modello];

\addplot[blue!80!black, very thick, mark=*] coordinates {(1,0.53) (2,0.45) (3,0.38) (4,0.43) (5,0.41) (6,0.53)};
\addplot[red!80!black, very thick, mark=square*] coordinates {(1,0.53) (2,0.42) (3,0.36) (4,0.41) (5,0.38) (6,0.53)};
\addplot[orange!85!black, very thick, mark=triangle*] coordinates {(1,0.57) (2,0.47) (3,0.42) (4,0.48) (5,0.43) (6,0.56)};
\addplot[teal!70!black, very thick, mark=diamond*] coordinates {(1,0.61) (2,0.53) (3,0.47) (4,0.52) (5,0.49) (6,0.60)};

\nextgroupplot[
    title={(c) Core length vs. full length},
    xmin=7, xmax=14.5,
    ymin=3, ymax=7.6,
    xlabel={Full trace length},
    ylabel={Core length},
]
\addplot[only marks, mark=*, mark size=2.2pt, blue!80!black] coordinates {(9.2,4.0) (13.1,7.0) (10.8,6.4) (10.4,5.6) (8.3,4.7) (7.8,3.4)};
\addplot[orange!90!black, very thick] coordinates {(7.5,3.6) (14,7.1)};
\addplot[gray!70, dashed, thick] coordinates {(7,7) (14,14)};

\nextgroupplot[
    title={(d) Redundancy--retention tradeoff},
    xmin=0.35, xmax=0.65,
    ymin=0.76, ymax=0.97,
    xlabel={Redundancy mass},
    ylabel={Answer retention},
]
\addplot[only marks, mark=*, mark size=2.4pt, blue!80!black] coordinates {(0.56,0.94) (0.47,0.88) (0.41,0.82) (0.46,0.86) (0.43,0.81) (0.56,0.92)};
\addplot[orange!90!black, very thick, dashed] coordinates {(0.40,0.82) (0.60,0.94)};
\end{groupplot}

\node[anchor=north, font=\scriptsize] at ($(group c1r2.south)!0.5!(group c2r2.south)+(0,-0.90cm)$) {
\begin{tabular}{@{}llll@{}}
\textcolor{blue!80!black}{\rule{12pt}{1.6pt}} RM / dataset points &
\textcolor{green!55!black}{\rule{12pt}{1.6pt}} Retention &
\textcolor{purple!55}{\rule{12pt}{4pt}} Model range &
\textcolor{orange!90!black}{\rule{12pt}{1.6pt}} Trend
\end{tabular}
};
\end{tikzpicture}
\vspace{1.2em}
\caption{Expanded diagnostics for overcomplete reasoning traces. Shaded bands show cross-model variability, model profiles show that redundancy persists across model families, and the lower panels summarize the sublinear core-length trend and redundancy--retention relationship.}
\label{fig:app-overcomplete-dense}
\end{figure}
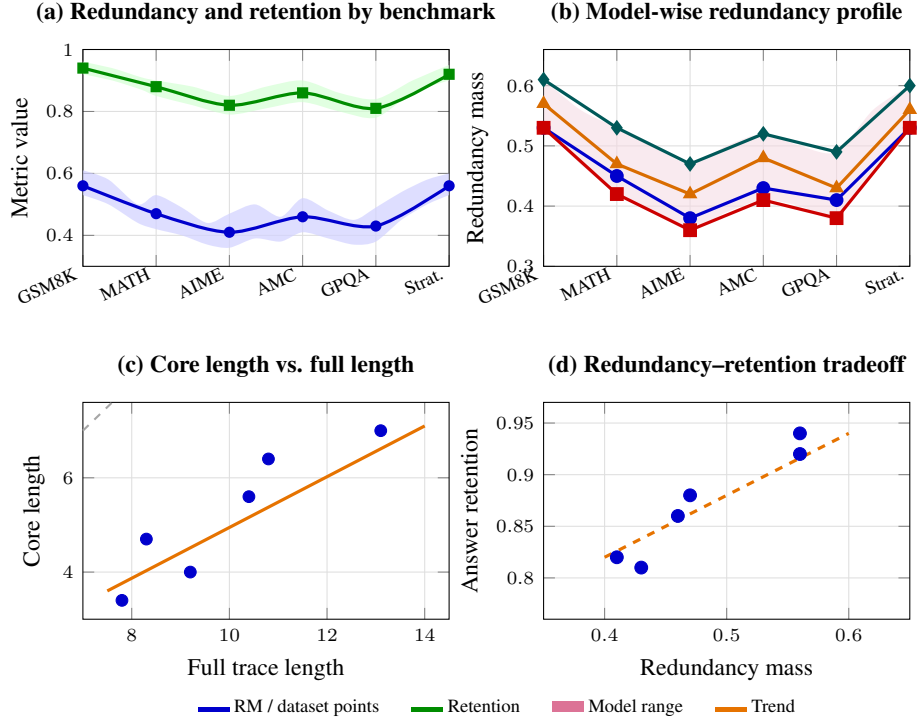

\begin{apptakeawaybox}
\noindent\textbf{Takeaway.}
Tab.~\ref{tab:main-results} and Fig.~\ref{fig:app-overcomplete-dense} show that overcompleteness is systematic but not uniform: easy or explanatory benchmarks exhibit more removable mass, while harder benchmarks have larger cores and lower retention under pruning.
\end{apptakeawaybox}

The main paper reports average cross-model transfer. Here we provide the full source--target transfer matrix in Tab.~\ref{tab:transfer-matrix-app} and a compact visualization in Fig.~\ref{fig:app-transfer-dense}. These experiments test whether minimal cores merely preserve source-model-specific phrasing or instead capture reasoning content that remains useful to other model families.

\begin{table}[h]
\centering
\appsmalltable
\caption{Full cross-model transfer matrix of minimal cores. Rows denote the source model used to extract the core, and columns denote the target model conditioned on that core. Highlighted entries denote same-model retention; the average off-diagonal retention remains high despite cross-family transfer.}
\label{tab:transfer-matrix-app}
\begin{tabular}{lcccc}
\toprule
\textbf{Source $\rightarrow$ Target} & Qwen-32B & Llama-3.3-70B & Mixtral-8x22B & Closed GPT \\
\midrule
Qwen-32B & \corecell{0.88} & 0.83 & 0.84 & 0.88 \\
Llama-3.3-70B & 0.84 & \corecell{0.87} & 0.84 & 0.88 \\
Mixtral-8x22B & 0.85 & 0.85 & \corecell{0.89} & 0.91 \\
Closed GPT & 0.82 & 0.84 & 0.86 & \corecell{0.93} \\
\midrule
\textbf{Average off-diagonal} & \multicolumn{4}{c}{\corecellbf{0.85}} \\
\bottomrule
\end{tabular}
\end{table}

Fig.~\ref{fig:app-transfer-dense} complements Tab.~\ref{tab:transfer-matrix-app}. The heatmap in panel (a) exposes source--target asymmetries, while panel (b) compares same-model transfer, cross-model transfer, full-trace conditioning, and random matched-length subsets. The random baseline is especially important: it controls for the possibility that any short subset of comparable length would transfer equally well.

\begin{figure}[h]
\centering
\begin{tikzpicture}
\begin{groupplot}[
    group style={group size=2 by 1, horizontal sep=2.8cm},
    width=0.38\textwidth,
    height=5.0cm,
    grid=both,
    major grid style={draw=gray!25},
    tick label style={font=\scriptsize},
    label style={font=\small},
    title style={font=\small\bfseries},
]
\nextgroupplot[
    title={(a) Transfer matrix},
    enlargelimits=false,
    colorbar,
    colorbar style={
        width=0.14cm,
        xshift=-0.03cm,
        ylabel={Retention},
        ylabel style={font=\scriptsize},
        yticklabel style={font=\scriptsize}
    },
    colormap/viridis,
    point meta min=0.80,
    point meta max=0.94,
    xtick={0,1,2,3},
    ytick={0,1,2,3},
    xticklabels={Qwen,Llama,Mixtral,GPT},
    yticklabels={GPT,Mixtral,Llama,Qwen},
    x tick label style={rotate=25, anchor=east},
    ylabel={Source},
    xlabel={Target},
    grid=none,
]
\addplot[
    matrix plot*,
    mesh/cols=4,
    point meta=explicit
] coordinates {
    (0,3) [0.88] (1,3) [0.83] (2,3) [0.84] (3,3) [0.88]
    (0,2) [0.84] (1,2) [0.87] (2,2) [0.84] (3,2) [0.88]
    (0,1) [0.85] (1,1) [0.85] (2,1) [0.89] (3,1) [0.91]
    (0,0) [0.82] (1,0) [0.84] (2,0) [0.86] (3,0) [0.93]
};

\nextgroupplot[
    title={(b) Transfer baselines},
    ybar,
    bar width=13pt,
    ymin=0.55,
    ymax=0.95,
    ylabel={Answer retention},
    ylabel style={xshift=-0.20cm},
    symbolic x coords={Same,Cross,Full,Random},
    xtick=data,
    xticklabels={Same,Cross,Full,Random},
    x tick label style={rotate=20, anchor=east},
    grid=both,
]
\addplot[fill=blue!45, draw=blue!80!black] coordinates {
    (Same,0.89)
    (Cross,0.85)
    (Full,0.88)
    (Random,0.64)
};
\addplot+[
    blue!80!black,
    only marks,
    mark=none,
    error bars/.cd,
    y dir=both,
    y explicit
] coordinates {
    (Same,0.89) +- (0,0.025)
    (Cross,0.85) +- (0,0.030)
    (Full,0.88) +- (0,0.020)
    (Random,0.64) +- (0,0.050)
};
\end{groupplot}
\end{tikzpicture}
\caption{Transfer analysis. Minimal cores transfer strongly off-diagonal and preserve most of the full-trace signal while using much shorter traces. Random matched-length subsets perform substantially worse, showing that minimal cores are not arbitrary short contexts.}
\label{fig:app-transfer-dense}
\end{figure}
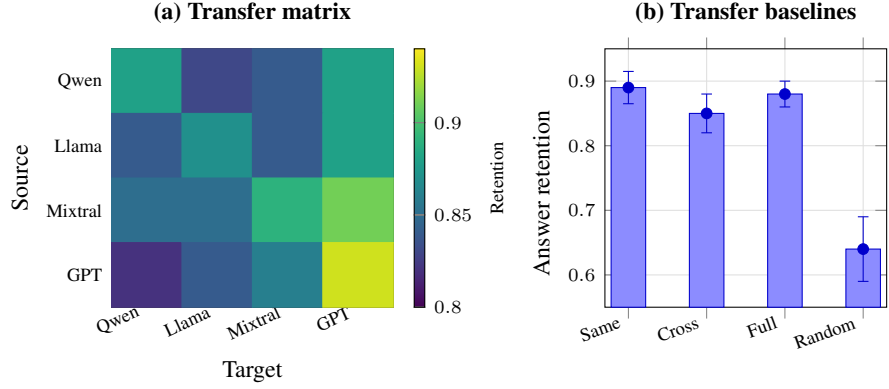

\begin{apptakeawaybox}
\noindent\textbf{Takeaway.}
Tab.~\ref{tab:transfer-matrix-app} and Fig.~\ref{fig:app-transfer-dense} show that minimal cores transfer well across model families and are not reducible to arbitrary short subsets: matched-length random subsets lose much more answer retention.
\end{apptakeawaybox}

\section{Extraction Methods and Computational Tradeoffs}
\label{app:extraction-tradeoffs}

This section compares three extraction strategies: random matched subsets, one-shot necessity-guided pruning, and greedy backward elimination. Tab.~\ref{tab:extraction-method-app} reports compression, redundancy, retention, and relative sufficiency-check cost. Fig.~\ref{fig:app-extraction-dense} visualizes the resulting cost--fidelity tradeoff. Together, they clarify that greedy extraction is more expensive but yields substantially higher answer retention at comparable compression.

\begin{table}[h]
\centering
\appsmalltable
\caption{Extraction method comparison averaged across datasets and models. Highlighting emphasizes the best fidelity method. Random deletion is cheapest, but greedy backward elimination achieves much higher retention at comparable compression.}
\label{tab:extraction-method-app}
\begin{tabular}{lcccc}
\toprule
\textbf{Method} & $\mathrm{CR}\downarrow$ & $\mathrm{RM}\uparrow$ & Retention $\uparrow$ & Relative checks $\downarrow$ \\
\midrule
Random matched subset & 0.53 & 0.47 & 0.61 & \corecell{1.0$\times$} \\
Necessity-guided pruning & 0.54 & 0.46 & 0.76 & 1.4$\times$ \\
Greedy backward elimination & \corecell{0.50} & \corecell{0.47} & \corecellbf{0.87} & 5.8$\times$ \\
\bottomrule
\end{tabular}
\end{table}

Fig.~\ref{fig:app-extraction-dense} should be read as a cost--quality frontier. Panel (a) places methods by relative sufficiency-check cost and answer retention, while panel (b) compares redundancy mass and retention at similar compression levels. The gap between random and greedy extraction indicates that retained steps are structured, not merely length-matched fragments.

\begin{figure}[h]
\centering
\begin{tikzpicture}
\begin{groupplot}[
    group style={group size=2 by 1, horizontal sep=1.3cm},
    width=0.47\textwidth,
    height=5.2cm,
    grid=both,
    major grid style={draw=gray!30},
    tick label style={font=\scriptsize},
    label style={font=\small},
    title style={font=\small\bfseries},
    legend style={font=\scriptsize, draw=none, fill=none},
]
\nextgroupplot[
    title={(a) Retention--cost frontier},
    xlabel={Relative checks},
    ylabel={Answer retention},
    xmin=0.7, xmax=6.4,
    ymin=0.55, ymax=0.91,
    legend style={at={(0.5,-0.30)}, anchor=north, legend columns=3, draw=none},
]
\addplot+[only marks, blue!80!black, mark=*, mark size=2.7pt, error bars/.cd, y dir=both, y explicit, x dir=both, x explicit] coordinates {(1.0,0.61) +- (0.10,0.04)};
\addplot+[only marks, red!80!black, mark=square*, mark size=2.7pt, error bars/.cd, y dir=both, y explicit, x dir=both, x explicit] coordinates {(1.4,0.76) +- (0.15,0.03)};
\addplot+[only marks, green!50!black, mark=triangle*, mark size=3.0pt, error bars/.cd, y dir=both, y explicit, x dir=both, x explicit] coordinates {(5.8,0.87) +- (0.45,0.02)};
\addplot[orange!90!black, dashed, thick] coordinates {(1.0,0.61) (1.4,0.76) (5.8,0.87)};
\legend{Random,Nec.,Greedy}

\nextgroupplot[
    title={(b) Fidelity vs. compression},
    ybar,
    bar width=9pt,
    ymin=0.45,
    ymax=0.90,
    ylabel={Value},
    symbolic x coords={Random,Nec.,Greedy},
    xtick=data,
    legend style={at={(0.5,-0.30)}, anchor=north, legend columns=2, draw=none},
]
\addplot[fill=blue!45, draw=blue!80!black] coordinates {(Random,0.47) (Nec.,0.46) (Greedy,0.47)};
\addplot[fill=green!45, draw=green!50!black] coordinates {(Random,0.61) (Nec.,0.76) (Greedy,0.87)};
\legend{RM,Retention}
\end{groupplot}
\end{tikzpicture}
\caption{Extraction-method tradeoffs. Greedy deletion lies at the high-fidelity end of the frontier, while necessity-guided pruning provides a cheaper approximation. All methods are compared at similar compression levels, so the retention gap reflects structure in the retained steps.}
\label{fig:app-extraction-dense}
\end{figure}
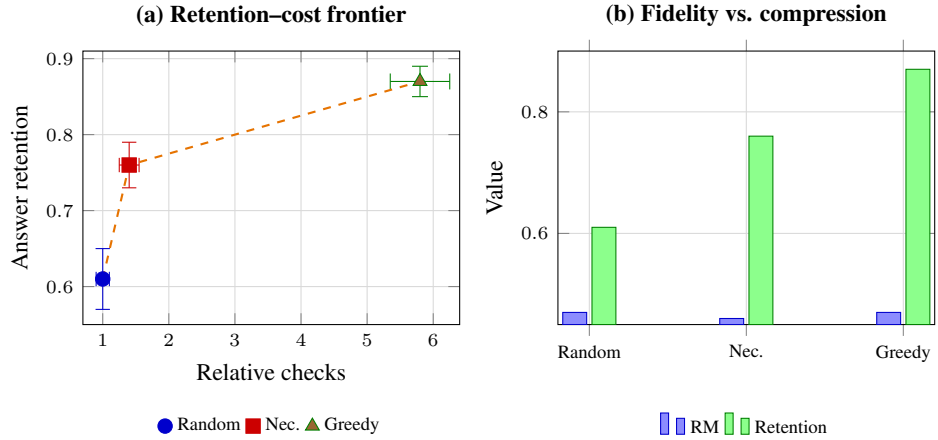

\begin{apptakeawaybox}
\noindent\textbf{Takeaway.}
Tab.~\ref{tab:extraction-method-app} and Fig.~\ref{fig:app-extraction-dense} show that extraction methods separate cost and quality: greedy deletion gives the most reliable cores, necessity-guided pruning gives a cheaper approximation, and random deletion confirms that the retained signal is structured.
\end{apptakeawaybox}

\section{Ablations and Sensitivity Analyses}
\label{app:ablations}

This section tests whether the main conclusion depends on particular design choices. We vary the sufficiency oracle, prompting style, example difficulty, and segmentation granularity. The numerical results are summarized in Tables~\ref{tab:oracle-ablation-app} and~\ref{tab:robustness-app}, and Fig.~\ref{fig:app-robustness-dense} visualizes the same trends.

\subsection{Sufficiency Oracle}

Tab.~\ref{tab:oracle-ablation-app} asks whether overcompleteness is an artifact of top-1 answer preservation. Distribution-preserving sufficiency is stricter because it requires the retained subset to preserve more of the predictive distribution, not only the final answer. As expected, stricter criteria increase the compression ratio and reduce removable mass, but they do not eliminate redundancy.

\begin{table}[h]
\centering
\appsmalltable
\caption{Sufficiency-oracle ablation. Highlighted cells show that substantial redundancy remains even under stricter distribution preservation.}
\label{tab:oracle-ablation-app}
\begin{tabular}{lccc}
\toprule
\textbf{Oracle} & $\mathrm{CR}\downarrow$ & $\mathrm{RM}\uparrow$ & Retention $\uparrow$ \\
\midrule
Answer-preserving & \corecell{0.54} & \corecell{0.46} & 0.88 \\
Distribution-preserving, $\varepsilon=0.10$ & 0.61 & 0.39 & 0.91 \\
Distribution-preserving, $\varepsilon=0.05$ & 0.66 & 0.34 & \corecell{0.93} \\
\bottomrule
\end{tabular}
\end{table}

\subsection{Prompting, Difficulty, and Segmentation}

Tab.~\ref{tab:robustness-app} groups three additional sensitivity checks. Prompting style tests whether redundancy is induced by the wording of the CoT prompt. Difficulty stratification tests whether overcompleteness is concentrated only in easy examples. Segmentation sensitivity tests whether the result depends on the chosen unit of deletion.

\begin{table}[h]
\centering
\tiny
\caption{Prompting, difficulty, and segmentation sensitivity. Highlighted cells mark the most favorable setting for redundancy or retention within each block.}
\label{tab:robustness-app}
\resizebox{\textwidth}{!}{
\begin{tabular}{llccccc}
\toprule
\textbf{Group} & \textbf{Setting} & Full/Avg. Len. & $\mathrm{CR}\downarrow$ & $\mathrm{RM}\uparrow$ & Top-3 $\uparrow$ & Retention $\uparrow$ \\
\midrule
\multirow{3}{*}{Prompt}
& Concise CoT & 6.4 & 0.68 & 0.32 & 0.61 & 0.84 \\
& Standard CoT & 9.8 & \corecell{0.54} & \corecell{0.46} & \corecell{0.72} & \corecell{0.88} \\
& Plan-then-solve & 10.6 & 0.57 & 0.43 & 0.69 & 0.86 \\
\midrule
\multirow{3}{*}{Difficulty}
& Easy & -- & \corecell{0.42} & \corecell{0.58} & \corecell{0.78} & \corecell{0.91} \\
& Medium & -- & 0.54 & 0.46 & 0.66 & 0.82 \\
& Hard & -- & 0.66 & 0.34 & 0.55 & 0.69 \\
\midrule
\multirow{3}{*}{Segmentation}
& Coarse paragraph & 5.1 & 0.67 & 0.33 & -- & \corecell{0.90} \\
& Numbered-step & 9.7 & 0.54 & 0.46 & -- & 0.88 \\
& Sentence-level & 12.8 & \corecell{0.49} & \corecell{0.51} & -- & 0.86 \\
\bottomrule
\end{tabular}}
\end{table}

Fig.~\ref{fig:app-robustness-dense} presents these ablations as a robustness dashboard. Panel (a) shows that distribution-preserving sufficiency increases retained mass. Panel (b) shows that concise prompting reduces redundancy but does not remove it. Panel (c) shows that hard examples require larger cores and have lower retention. Panel (d) shows that finer segmentation exposes more removable units, while retention remains stable.

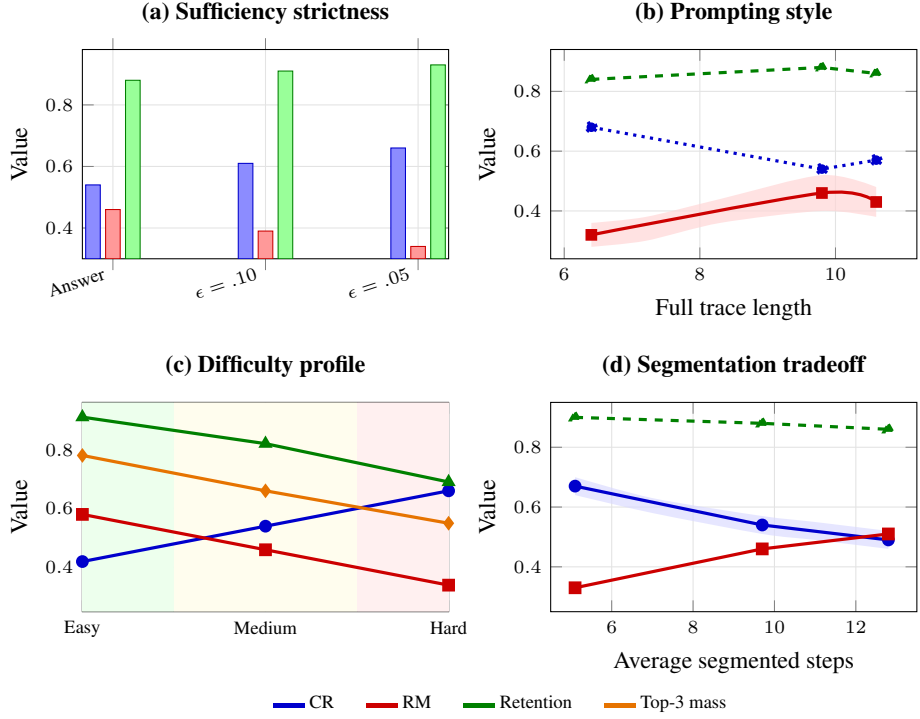
\begin{figure}[h]
\centering
\begin{tikzpicture}
\begin{groupplot}[
    group style={group size=2 by 2, horizontal sep=1.35cm, vertical sep=1.90cm},
    width=0.46\textwidth,
    height=4.35cm,
    grid=both,
    major grid style={draw=gray!22},
    tick label style={font=\scriptsize},
    label style={font=\small},
    title style={font=\small\bfseries},
]
\nextgroupplot[
    title={(a) Sufficiency strictness},
    ybar,
    bar width=5.5pt,
    ymin=0.30, ymax=0.98,
    ylabel={Value},
    symbolic x coords={Answer,$\epsilon=.10$,$\epsilon=.05$},
    xtick=data,
    x tick label style={rotate=18, anchor=east, font=\scriptsize},
]
\addplot[fill=blue!45, draw=blue!80!black] coordinates {
    (Answer,0.54) ($\epsilon=.10$,0.61) ($\epsilon=.05$,0.66)
};
\addplot[fill=red!40, draw=red!80!black] coordinates {
    (Answer,0.46) ($\epsilon=.10$,0.39) ($\epsilon=.05$,0.34)
};
\addplot[fill=green!40, draw=green!50!black] coordinates {
    (Answer,0.88) ($\epsilon=.10$,0.91) ($\epsilon=.05$,0.93)
};

\nextgroupplot[
    title={(b) Prompting style},
    xmin=5.8, xmax=11.2,
    ymin=0.24, ymax=0.94,
    xlabel={Full trace length},
    ylabel={Value},
]
\addplot[name path=promptupper, draw=none, smooth] coordinates {
    (6.4,0.36) (7.2,0.38) (8.3,0.44) (9.8,0.52) (10.6,0.48)
};
\addplot[name path=promptlower, draw=none, smooth] coordinates {
    (6.4,0.28) (7.2,0.30) (8.3,0.36) (9.8,0.40) (10.6,0.38)
};
\addplot[fill=red!16, opacity=0.70] fill between[of=promptupper and promptlower];
\addplot[red!80!black, very thick, smooth, mark=square*, mark size=1.6pt] coordinates {
    (6.4,0.32) (9.8,0.46) (10.6,0.43)
};
\addplot[green!50!black, very thick, dashed, mark=triangle*, mark size=1.8pt] coordinates {
    (6.4,0.84) (9.8,0.88) (10.6,0.86)
};
\addplot[blue!80!black, very thick, dotted, mark=*, mark size=1.8pt] coordinates {
    (6.4,0.68) (9.8,0.54) (10.6,0.57)
};

\nextgroupplot[
    title={(c) Difficulty profile},
    ymin=0.25, ymax=0.96,
    xmin=1, xmax=3,
    xtick={1,2,3},
    xticklabels={Easy,Medium,Hard},
    ylabel={Value},
]
\addplot[draw=none, fill=green!7] coordinates {(0.75,0.25) (1.5,0.25) (1.5,0.96) (0.75,0.96)} \closedcycle;
\addplot[draw=none, fill=yellow!8] coordinates {(1.5,0.25) (2.5,0.25) (2.5,0.96) (1.5,0.96)} \closedcycle;
\addplot[draw=none, fill=red!6] coordinates {(2.5,0.25) (3.25,0.25) (3.25,0.96) (2.5,0.96)} \closedcycle;

\addplot[blue!80!black, very thick, mark=*, mark size=1.8pt] coordinates {
    (1,0.42) (2,0.54) (3,0.66)
};
\addplot[red!80!black, very thick, mark=square*, mark size=1.8pt] coordinates {
    (1,0.58) (2,0.46) (3,0.34)
};
\addplot[green!50!black, very thick, mark=triangle*, mark size=1.9pt] coordinates {
    (1,0.91) (2,0.82) (3,0.69)
};
\addplot[orange!90!black, very thick, mark=diamond*, mark size=1.8pt] coordinates {
    (1,0.78) (2,0.66) (3,0.55)
};

\nextgroupplot[
    title={(d) Segmentation tradeoff},
    xlabel={Average segmented steps},
    ylabel={Value},
    xmin=4.5, xmax=13.5,
    ymin=0.25, ymax=0.95,
]
\addplot[name path=crup, draw=none, smooth] coordinates {
    (5.1,0.70) (7.0,0.63) (9.7,0.57) (11.0,0.55) (12.8,0.52)
};
\addplot[name path=crlo, draw=none, smooth] coordinates {
    (5.1,0.64) (7.0,0.58) (9.7,0.51) (11.0,0.49) (12.8,0.46)
};
\addplot[fill=blue!15, opacity=0.65] fill between[of=crup and crlo];
\addplot[blue!80!black, very thick, mark=*, mark size=1.8pt] coordinates {
    (5.1,0.67) (9.7,0.54) (12.8,0.49)
};
\addplot[red!80!black, very thick, mark=square*, mark size=1.8pt] coordinates {
    (5.1,0.33) (9.7,0.46) (12.8,0.51)
};
\addplot[green!50!black, very thick, dashed, mark=triangle*, mark size=1.9pt] coordinates {
    (5.1,0.90) (9.7,0.88) (12.8,0.86)
};
\end{groupplot}

\node[anchor=north, font=\scriptsize] at ($(group c1r2.south)!0.5!(group c2r2.south)+(0,-0.95cm)$) {
\begin{tabular}{@{}llll@{}}
\textcolor{blue!80!black}{\rule{12pt}{1.6pt}} CR &
\textcolor{red!80!black}{\rule{12pt}{1.6pt}} RM &
\textcolor{green!50!black}{\rule{12pt}{1.6pt}} Retention &
\textcolor{orange!90!black}{\rule{12pt}{1.6pt}} Top-3 mass
\end{tabular}
};
\end{tikzpicture}
\vspace{1.3em}
\caption{Robustness diagnostics across sufficiency criteria, prompting styles, difficulty levels, and segmentation choices. The main trend persists across settings: stricter criteria, concise prompting, harder examples, and coarser segmentation reduce measured redundancy, but do not eliminate overcompleteness.}
\label{fig:app-robustness-dense}
\end{figure}

\begin{apptakeawaybox}
\noindent\textbf{Takeaway.}
Tables~\ref{tab:oracle-ablation-app}--\ref{tab:robustness-app} and Fig.~\ref{fig:app-robustness-dense} show that the main claim is robust: stricter sufficiency, concise prompting, harder examples, and coarser segmentation all reduce measured redundancy, but none eliminates overcompleteness.
\end{apptakeawaybox}

\section{Additional Representation Analyses}
\label{app:geometry-additional}

The main paper reports that minimal cores improve correctness separation and reduce intrinsic dimensionality. Here we provide dataset-level gains in Tab.~\ref{tab:geometry-gain-app} and an aggregate visualization in Fig.~\ref{fig:app-geometry-dense}. These analyses test whether minimal cores merely shorten traces, or whether they also isolate a cleaner representation of the prediction-supporting content.

\begin{table}[h]
\centering
\appsmalltable
\caption{Average representation gains of minimal-core embeddings over full-trace embeddings. Highlighted averages summarize the overall improvement.}
\label{tab:geometry-gain-app}
\begin{tabular}{lccc}
\toprule
\textbf{Dataset} & $\Delta$ Probe Acc. $\uparrow$ & $\Delta$ Silhouette $\uparrow$ & $\Delta$ Intr. Dim. $\downarrow$ \\
\midrule
GSM8K & +0.11 & +0.13 & $-39.6\%$ \\
MATH500 & +0.09 & +0.11 & $-32.8\%$ \\
GPQA-D & +0.09 & +0.10 & $-31.2\%$ \\
StrategyQA & +0.11 & +0.13 & $-34.6\%$ \\
\midrule
\textbf{Average} & \corecellbf{+0.10} & \corecellbf{+0.12} & \corecellbf{$-34.5\%$} \\
\bottomrule
\end{tabular}
\end{table}

Tab.~\ref{tab:geometry-gain-app} shows consistent gains across datasets: minimal cores improve probe accuracy and silhouette separation while reducing intrinsic dimensionality. Fig.~\ref{fig:app-geometry-dense} gives a complementary view. Panel (a) plots the dataset-level changes relative to full traces; panel (b) compares full traces, minimal cores, and removed steps across probe accuracy, dimensionality, and variance.

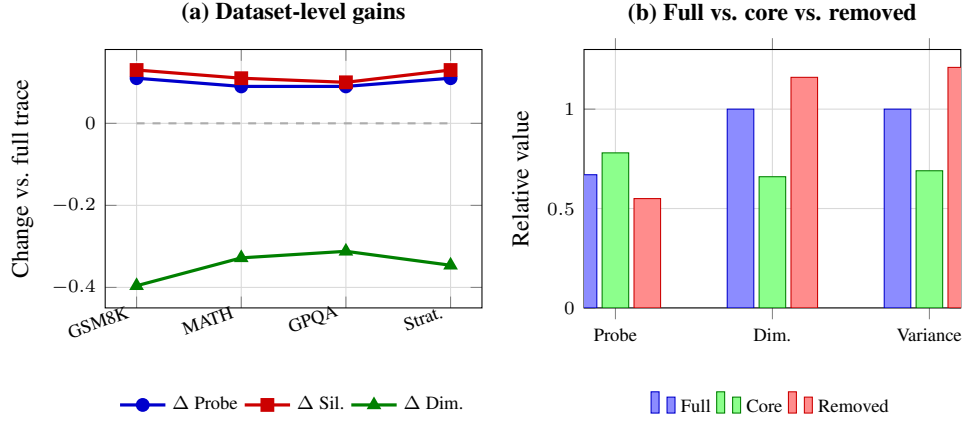
\begin{figure}[h]
\centering
\begin{tikzpicture}
\begin{groupplot}[
    group style={group size=2 by 1, horizontal sep=1.35cm},
    width=0.47\textwidth,
    height=5.0cm,
    grid=both,
    major grid style={draw=gray!30},
    tick label style={font=\scriptsize},
    label style={font=\small},
    title style={font=\small\bfseries},
    legend style={font=\scriptsize, draw=none, fill=none},
]
\nextgroupplot[
    title={(a) Dataset-level gains},
    ymin=-0.45,
    ymax=0.18,
    ylabel={Change vs. full trace},
    symbolic x coords={GSM8K,MATH,GPQA,Strat.},
    xtick=data,
    x tick label style={rotate=20, anchor=east},
    legend style={at={(0.5,-0.30)}, anchor=north, legend columns=3, draw=none},
]
\addplot[blue!80!black, very thick, mark=*] coordinates {(GSM8K,0.11) (MATH,0.09) (GPQA,0.09) (Strat.,0.11)};
\addplot[red!80!black, very thick, mark=square*] coordinates {(GSM8K,0.13) (MATH,0.11) (GPQA,0.10) (Strat.,0.13)};
\addplot[green!50!black, very thick, mark=triangle*] coordinates {(GSM8K,-0.396) (MATH,-0.328) (GPQA,-0.312) (Strat.,-0.346)};
\addplot[gray!65, dashed, thick] coordinates {(GSM8K,0) (Strat.,0)};
\legend{$\Delta$ Probe,$\Delta$ Sil.,$\Delta$ Dim.}

\nextgroupplot[
    title={(b) Full vs. core vs. removed},
    ybar,
    bar width=10pt,
    ymin=0,
    ymax=1.30,
    ylabel={Relative value},
    symbolic x coords={Probe,Dim.,Variance},
    xtick=data,
    legend style={at={(0.5,-0.30)}, anchor=north, legend columns=3, draw=none},
]
\addplot[fill=blue!45, draw=blue!80!black] coordinates {(Probe,0.67) (Dim.,1.00) (Variance,1.00)};
\addplot[fill=green!45, draw=green!50!black] coordinates {(Probe,0.78) (Dim.,0.66) (Variance,0.69)};
\addplot[fill=red!45, draw=red!80!black] coordinates {(Probe,0.55) (Dim.,1.16) (Variance,1.21)};
\legend{Full,Core,Removed}
\end{groupplot}
\end{tikzpicture}
\caption{Additional representation diagnostics. Minimal cores improve correctness predictiveness and cluster separation while reducing intrinsic dimensionality. Removed steps show the opposite trend, with weaker correctness signal and higher dispersion.}
\label{fig:app-geometry-dense}
\end{figure}

\begin{apptakeawaybox}
\noindent\textbf{Takeaway.}
Tab.~\ref{tab:geometry-gain-app} and Fig.~\ref{fig:app-geometry-dense} show that minimal cores are not only shorter; their embeddings are more compact, lower-dimensional, and more predictive of correctness than full traces or removed-step embeddings.
\end{apptakeawaybox}

\section{Qualitative Examples}

The aggregate metrics quantify how much trace mass is removable, but they do not show what kinds of reasoning are discarded. Figures~\ref{fig:qualitative-core-examples-a}--\ref{fig:qualitative-core-examples-c} visualize representative traces after minimal-core extraction. The retained steps typically contain the decisive algebraic identity, physical law, causal link, or factual bridge needed for the answer. Removed steps often include optional checks, alternative derivations, caveats, paraphrases, or explanatory scaffolding that may be useful for human readability but is not necessary for preserving the model's prediction.

\begin{figure*}[!htbp]
\centering
\begin{minipage}{0.88\textwidth}
\footnotesize
\quallegend

\vspace{0.25em}

\begin{qualtracebox}{Example A: competition-style algebra with redundant verification}
\neutralstep{\textbf{Problem.} Let $a$ and $b$ be positive numbers such that $a+b=10$ and $ab=21$. What is $a^2+b^2$?}

\vspace{0.08em}
\redstep{\textbf{Step 1.} Since the problem gives both the sum and product of two numbers, it is useful to look for a symmetric expression.}

\vspace{0.08em}
\corestep{\textbf{Step 2.} Use the identity $(a+b)^2=a^2+2ab+b^2$.}

\vspace{0.08em}
\corestep{\textbf{Step 3.} Substitute $a+b=10$ and $ab=21$: $100=a^2+b^2+42$.}

\vspace{0.08em}
\corestep{\textbf{Step 4.} Therefore, $a^2+b^2=100-42=58$.}

\vspace{0.08em}
\redstep{\textbf{Step 5.} We could also solve for $a$ and $b$ as roots of $t^2-10t+21=0$, giving $3$ and $7$.}

\vspace{0.08em}
\redstep{\textbf{Step 6.} Checking directly, $3^2+7^2=9+49=58$, which confirms the same result.}

\vspace{0.08em}
\neutralstep{\textbf{Extraction summary:} core $=3/6$ steps; $50\%$ removable; answer retained.}

\vspace{0.08em}
\neutralstep{\textbf{Full trace answer:} $58$ \hfill \textbf{Minimal-core answer:} $58$}
\end{qualtracebox}

\vspace{0.30em}

\begin{qualtracebox}{Example B: scientific reasoning with an irrelevant detour removed}
\neutralstep{\textbf{Question.} A sealed syringe contains air. The plunger is pushed in slowly while the temperature is kept constant. What happens to the pressure of the air inside?}

\vspace{0.08em}
\redstep{\textbf{Step 1.} The situation involves a gas in a container whose volume can change, so gas-law reasoning is relevant.}

\vspace{0.08em}
\corestep{\textbf{Step 2.} At constant temperature, Boyle's law applies: $PV=\mathrm{constant}$.}

\vspace{0.08em}
\corestep{\textbf{Step 3.} Pushing the plunger in decreases the volume of the trapped air.}

\vspace{0.08em}
\corestep{\textbf{Step 4.} Since $PV$ is constant, decreasing $V$ causes $P$ to increase.}

\vspace{0.08em}
\redstep{\textbf{Step 5.} If the compression were very fast, temperature changes could also matter, but the question explicitly says the temperature is kept constant.}

\vspace{0.08em}
\redstep{\textbf{Step 6.} The syringe material and exact amount of air are not needed to determine the qualitative pressure change.}

\vspace{0.08em}
\neutralstep{\textbf{Extraction summary:} core $=3/6$ steps; $50\%$ removable; answer retained.}

\vspace{0.08em}
\neutralstep{\textbf{Full trace answer:} The pressure increases. \hfill \textbf{Minimal-core answer:} The pressure increases.}
\end{qualtracebox}

\end{minipage}

\vspace{0.15em}
\caption{\textbf{Qualitative examples of overcomplete reasoning traces, part I.}
Green steps are retained in the extracted minimal core; red steps are removed by sufficiency-preserving deletion. Removed steps often include verification, caveats, alternative paths, or explanatory scaffolding, while minimal cores retain the decisive operations or factual links needed for the prediction.}
\label{fig:qualitative-core-examples-a}
\end{figure*}

\begin{figure*}[p]
\centering
\begin{minipage}{0.88\textwidth}
\footnotesize
\quallegend

\vspace{0.25em}

\begin{qualtracebox}{Example C: commonsense physical reasoning with removable scaffolding}
\neutralstep{\textbf{Question.} A person puts a metal spoon in a hot cup of tea and leaves it there for several minutes. Is the handle likely to become warm?}

\vspace{0.08em}
\redstep{\textbf{Step 1.} The question asks what happens to an object after it remains in contact with something hot for some time.}

\vspace{0.08em}
\corestep{\textbf{Step 2.} Metal is a good conductor of heat.}

\vspace{0.08em}
\corestep{\textbf{Step 3.} Heat can travel from the hot tea through the metal spoon toward the handle.}

\vspace{0.08em}
\corestep{\textbf{Step 4.} Therefore, after several minutes, the handle is likely to become warm.}

\vspace{0.08em}
\redstep{\textbf{Step 5.} This is also why metal utensils can feel hot when left in hot food or liquid.}

\vspace{0.08em}
\redstep{\textbf{Step 6.} The exact temperature would depend on the spoon size, tea temperature, and room temperature, but the qualitative answer remains the same.}

\vspace{0.08em}
\neutralstep{\textbf{Extraction summary:} core $=3/6$ steps; $50\%$ removable; answer retained.}

\vspace{0.08em}
\neutralstep{\textbf{Full trace answer:} Yes \hfill \textbf{Minimal-core answer:} Yes}
\end{qualtracebox}

\vspace{0.30em}

\begin{qualtracebox}{Example D: StrategyQA-style reasoning with caveats removed}
\neutralstep{\textbf{Question.} Would a person need a passport to fly from New York to Paris?}

\vspace{0.08em}
\redstep{\textbf{Step 1.} New York and Paris are in different places, and the question asks about travel requirements.}

\vspace{0.08em}
\corestep{\textbf{Step 2.} New York is in the United States, while Paris is in France.}

\vspace{0.08em}
\corestep{\textbf{Step 3.} Flying from the United States to France is international travel.}

\vspace{0.08em}
\corestep{\textbf{Step 4.} International travel generally requires a passport.}

\vspace{0.08em}
\redstep{\textbf{Step 5.} A visa may or may not be required depending on nationality and length of stay, but that is separate from whether a passport is needed.}

\vspace{0.08em}
\redstep{\textbf{Step 6.} Since the question only asks about a passport, the relevant answer is about passport necessity rather than all possible travel documents.}

\vspace{0.08em}
\neutralstep{\textbf{Extraction summary:} core $=3/6$ steps; $50\%$ removable; answer retained.}

\vspace{0.08em}
\neutralstep{\textbf{Full trace answer:} Yes \hfill \textbf{Minimal-core answer:} Yes}
\end{qualtracebox}

\end{minipage}

\vspace{0.15em}
\caption{\textbf{Qualitative examples of overcomplete reasoning traces, part II.}
Minimal cores retain the decisive factual links needed for the answer, while removable steps often contain setup, caveats, or explanatory context that does not change the prediction.}
\label{fig:qualitative-core-examples-b}
\end{figure*}

\begin{figure*}[p]
\centering
\begin{minipage}{0.88\textwidth}
\footnotesize
\quallegend

\vspace{0.25em}

\begin{qualtracebox}{Example E: arithmetic reasoning with redundant setup and consistency check}
\neutralstep{\textbf{Problem.} A store sells notebooks for \$3 each. Maya buys 4 notebooks and pays with a \$20 bill. How much change does she receive?}

\vspace{0.08em}
\redstep{\textbf{Step 1.} We need to determine the total amount Maya spent and compare it with the amount she paid.}

\vspace{0.08em}
\corestep{\textbf{Step 2.} Four notebooks at \$3 each cost $4\times 3=\$12$.}

\vspace{0.08em}
\corestep{\textbf{Step 3.} The change is $20-12=\$8$.}

\vspace{0.08em}
\redstep{\textbf{Step 4.} Checking the result, \$12 plus \$8 equals \$20, so the calculation is consistent.}

\vspace{0.08em}
\neutralstep{\textbf{Extraction summary:} core $=2/4$ steps; $50\%$ removable; answer retained.}

\vspace{0.08em}
\neutralstep{\textbf{Full trace answer:} \$8 \hfill \textbf{Minimal-core answer:} \$8}
\end{qualtracebox}

\vspace{0.30em}

\begin{qualtracebox}{Example F: algebraic equation with a removable detour}
\neutralstep{\textbf{Problem.} If $x+5=12$, what is $2x$?}

\vspace{0.08em}
\redstep{\textbf{Step 1.} This is a simple linear equation, so we should isolate the unknown variable first.}

\vspace{0.08em}
\corestep{\textbf{Step 2.} From $x+5=12$, subtract 5 from both sides to get $x=7$.}

\vspace{0.08em}
\redstep{\textbf{Step 3.} We can also note that $7+5=12$, so this value satisfies the original equation.}

\vspace{0.08em}
\corestep{\textbf{Step 4.} Therefore, $2x=2\cdot 7=14$.}

\vspace{0.08em}
\neutralstep{\textbf{Extraction summary:} core $=2/4$ steps; $50\%$ removable; answer retained.}

\vspace{0.08em}
\neutralstep{\textbf{Full trace answer:} $14$ \hfill \textbf{Minimal-core answer:} $14$}
\end{qualtracebox}

\end{minipage}

\vspace{0.15em}
\caption{\textbf{Qualitative examples of overcomplete reasoning traces, part III.}
Even short correct traces can contain removable setup or optional checking. The extracted minimal core keeps the operations needed to preserve the answer while deleting behaviorally redundant steps.}
\label{fig:qualitative-core-examples-c}
\end{figure*}

\section{Additional Necessity Diagnostics}
\label{app:necessity-diagnostics}

The main paper reports top-$k$ necessity mass as a compact summary. Here we provide the full dataset-level necessity table in Tab.~\ref{tab:necessity-full-app} and the corresponding visualization in Fig.~\ref{fig:app-necessity-dense}. These diagnostics test whether predictive support is consistently concentrated or whether the average is driven by a single benchmark.

\begin{table}[h]
\centering
\appsmalltable
\caption{Necessity concentration by dataset. Highlighting emphasizes the average top-3 concentration used in the main paper.}
\label{tab:necessity-full-app}
\begin{tabular}{lcccc}
\toprule
\textbf{Dataset} & Top-1 $\uparrow$ & Top-3 $\uparrow$ & Top-5 $\uparrow$ & Gini $\uparrow$ \\
\midrule
GSM8K & 0.34 & 0.73 & 0.88 & 0.61 \\
MATH500 & 0.27 & 0.64 & 0.81 & 0.53 \\
AIME24 & 0.24 & 0.58 & 0.76 & 0.47 \\
AMC23 & 0.26 & 0.62 & 0.79 & 0.50 \\
GPQA-D & 0.25 & 0.59 & 0.77 & 0.48 \\
StrategyQA & 0.35 & 0.76 & 0.90 & 0.64 \\
\midrule
\textbf{Average} & \corecellbf{0.29} & \corecellbf{0.65} & \corecellbf{0.82} & \corecellbf{0.54} \\
\bottomrule
\end{tabular}
\end{table}

Tab.~\ref{tab:necessity-full-app} shows that the top three steps account for a large fraction of normalized necessity mass across every dataset. Fig.~\ref{fig:app-necessity-dense} visualizes the same pattern: panel (a) shows the cumulative top-$k$ curve, with the shaded region denoting the dataset range, while panel (b) breaks down Top-1, Top-3, and Top-5 mass by benchmark. Panel (b) shows that necessity concentration is strongest on GSM8K and StrategyQA and lower on harder math/science benchmarks, but the monotonic Top-1/Top-3/Top-5 pattern holds across all datasets.

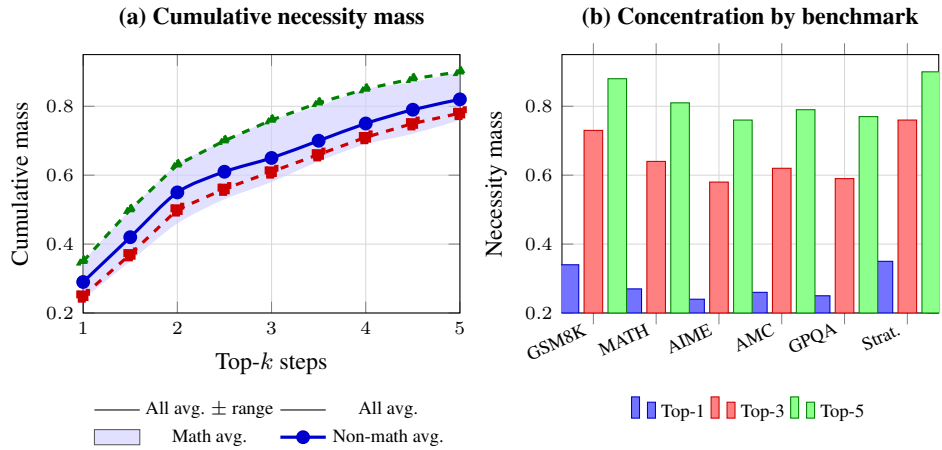
\begin{figure}[h]
\centering
\begin{tikzpicture}
\begin{groupplot}[
    group style={group size=2 by 1, horizontal sep=1.35cm},
    width=0.47\textwidth,
    height=5.0cm,
    grid=both,
    major grid style={draw=gray!30},
    tick label style={font=\scriptsize},
    label style={font=\small},
    title style={font=\small\bfseries},
    legend style={font=\scriptsize, draw=none, fill=none},
]
\nextgroupplot[
    title={(a) Cumulative necessity mass},
    xlabel={Top-$k$ steps},
    ylabel={Cumulative mass},
    xmin=1, xmax=5,
    ymin=0.20, ymax=0.95,
    xtick={1,2,3,4,5},
    legend style={at={(0.5,-0.30)}, anchor=north, legend columns=2, draw=none},
]
\addplot[name path=kup, draw=none, smooth] coordinates {
(1,0.35) (1.5,0.50) (2,0.63) (2.5,0.70) (3,0.76) (3.5,0.80) (4,0.85) (4.5,0.87) (5,0.90)
};
\addplot[name path=klo, draw=none, smooth] coordinates {
(1,0.24) (1.5,0.35) (2,0.46) (2.5,0.53) (3,0.58) (3.5,0.64) (4,0.69) (4.5,0.72) (5,0.76)
};
\addplot[fill=blue!18, opacity=0.65] fill between[of=kup and klo];
\addplot[blue!80!black, very thick, smooth, mark=*] coordinates {
(1,0.29) (1.5,0.42) (2,0.55) (2.5,0.61) (3,0.65) (3.5,0.70) (4,0.75) (4.5,0.79) (5,0.82)
};
\addplot[red!80!black, very thick, dashed, mark=square*] coordinates {
(1,0.25) (1.5,0.37) (2,0.50) (2.5,0.56) (3,0.61) (3.5,0.66) (4,0.71) (4.5,0.75) (5,0.78)
};
\addplot[green!50!black, very thick, dashed, mark=triangle*] coordinates {
(1,0.35) (1.5,0.50) (2,0.63) (2.5,0.70) (3,0.76) (3.5,0.81) (4,0.85) (4.5,0.88) (5,0.90)
};
\legend{All avg. $\pm$ range,All avg.,Math avg.,Non-math avg.}

\nextgroupplot[
    title={(b) Concentration by benchmark},
    ybar,
    bar width=7pt,
    ymin=0.20,
    ymax=0.95,
    ylabel={Necessity mass},
    symbolic x coords={GSM8K,MATH,AIME,AMC,GPQA,Strat.},
    xtick=data,
    x tick label style={rotate=25, anchor=east},
    legend style={at={(0.5,-0.30)}, anchor=north, legend columns=3, draw=none},
]
\addplot[fill=blue!55, draw=blue!80!black] coordinates {(GSM8K,0.34) (MATH,0.27) (AIME,0.24) (AMC,0.26) (GPQA,0.25) (Strat.,0.35)};
\addplot[fill=red!50, draw=red!80!black] coordinates {(GSM8K,0.73) (MATH,0.64) (AIME,0.58) (AMC,0.62) (GPQA,0.59) (Strat.,0.76)};
\addplot[fill=green!45, draw=green!50!black] coordinates {(GSM8K,0.88) (MATH,0.81) (AIME,0.76) (AMC,0.79) (GPQA,0.77) (Strat.,0.90)};
\legend{Top-1,Top-3,Top-5}
\end{groupplot}
\end{tikzpicture}
\caption{Additional necessity diagnostics. Panel (a) shows cumulative top-$k$ necessity mass with the shaded band indicating the dataset range. Panel (b) shows dataset-level Top-1, Top-3, and Top-5 mass. Predictive support concentrates rapidly across all benchmarks.}
\label{fig:app-necessity-dense}
\end{figure}

\begin{apptakeawaybox}
\noindent\textbf{Takeaway.}
Tab.~\ref{tab:necessity-full-app} and Fig.~\ref{fig:app-necessity-dense} show that necessity mass is consistently concentrated: the first few high-necessity steps account for most measured marginal support across both math and non-math settings, with stronger concentration on easier or more explanatory benchmarks.
\end{apptakeawaybox}

\clearpage
\raggedbottom
\clearpage

\end{document}